%% file: main.tex
\documentclass{article}

\usepackage[preprint]{neurips_2026}

\usepackage[utf8]{inputenc}
\usepackage[T1]{fontenc}
\usepackage{hyperref}
\usepackage{url}
\usepackage{booktabs}
\usepackage{amsfonts}
\usepackage{amsmath,amssymb}
\usepackage{amsthm}
\usepackage{mathrsfs}
\usepackage{mathtools}
\usepackage{nicefrac}
\usepackage{microtype}
\usepackage{xcolor}
\usepackage{graphicx}
\usepackage{tikz}
\usepackage{multirow}
\usepackage{float}
\usepackage{wrapfig}
\usepackage{caption}
\usepackage{etoc}
\usepackage{algorithm}
\usepackage{algorithmicx}
\usepackage{algpseudocode}
\usepackage{listings}
\usepackage{colortbl}
\usepackage{xspace}
\usepackage{array}

\newtheorem{theorem}{Theorem}
\newtheorem{lemma}[theorem]{Lemma}
\newtheorem{corollary}[theorem]{Corollary}
\newtheorem{proposition}[theorem]{Proposition}
\newtheorem{assumption}[theorem]{Assumption}
\newtheorem{definition}{Definition}
\newtheorem{remark}{Remark}

\newcommand{\modelname}{\textsc{WMSS}\xspace}
\newif\ifnips\nipstrue

\title{Weak-Driven Learning: How Weak Agents Make Strong Agents Stronger}

\author{%
  Zehao Chen$^{1,2,*}$ \quad
  Gongxun Li$^{1,2,*}$ \quad
  Tianxiang Ai$^{2,*}$ \quad
  Yifei Li$^{1,2}$ \\
  Zixuan Huang$^{1}$ \quad
  Wang Zhou$^{2}$ \quad
  Tao Huang$^{2}$ \quad
  Fuzhen Zhuang$^{1}$ \\
  Xianglong Liu$^{1}$ \quad
  Jianxin Li$^{1}$ \quad
  Deqing Wang$^{1,\dagger}$ \quad
  Yikun Ban$^{1,\dagger}$ \\
  $^{1}$Beihang University \quad
  $^{2}$China Telecom eSurfing Cloud \\
  $^{*}$Equal contribution \quad
  $^{\dagger}$Corresponding authors \\
  \texttt{gyy\_chenzehao@chinatelecom.cn} \quad
  \texttt{zehaochenacid@buaa.edu.cn} \quad
  \texttt{yikunb@buaa.edu.cn}
}

\begin{document}

\maketitle

\input{content/0_abstract}

\input{content/1_introduction}

\input{content/2_related_work}

\input{content/3_method}

\input{content/4_experiments}

\input{content/5_conclusion}

\newpage
\bibliographystyle{plainnat}
\bibliography{references}

\newpage
\appendix
\section*{Appendix Contents}
\begin{small}
\begin{enumerate}
    \item[A.] \hyperref[app:limitations]{Limitations and Broader Implications}\hfill \pageref{app:limitations}
    \item[B.] \hyperref[appendix: additionalexperiments]{Additional Experiments}\hfill \pageref{appendix: additionalexperiments}
    \begin{enumerate}
        \item[B.1] \hyperref[app:frozen_weak_agent]{Freezing the Weak Agent during Weak-Driven Training}\hfill \pageref{app:frozen_weak_agent}
        \item[B.2] \hyperref[app:more_models]{Results on Gemma-3-4B-PT}\hfill \pageref{app:more_models}
        \item[B.3] \hyperref[app:logic_results]{Logic-Reasoning Results}\hfill \pageref{app:logic_results}
        \item[B.4] \hyperref[sec:logit_statistics]{Logit Statistics}\hfill \pageref{sec:logit_statistics}
        \item[B.5] \hyperref[app:checkpoint_pairing]{Weak--Strong Entropy Separation}\hfill \pageref{app:checkpoint_pairing}
        \item[B.6] \hyperref[app:hparam_protocol]{Hyperparameter Sensitivity}\hfill \pageref{app:hparam_protocol}
        \item[B.7] \hyperref[app:pass_at_k]{Pass@\textit{k} Comparison with DPO}\hfill \pageref{app:pass_at_k}
        \item[B.8] \hyperref[app:wd_ds_runtime]{Training Cost and WD-DS Runtime}\hfill \pageref{app:wd_ds_runtime}
    \end{enumerate}
    \item[C.] \hyperref[sec:appendix_grad_amp]{Supplementary Theory: Gradient Amplification}\hfill \pageref{sec:appendix_grad_amp}
    \item[D.] \hyperref[sec:appendix_joint_theory]{Additional Dynamics Diagnostics}\hfill \pageref{sec:appendix_joint_theory}
    \item[E.] \hyperref[app:dataconstruction]{Data Construction and Training Details}\hfill \pageref{app:dataconstruction}
\end{enumerate}
\end{small}
\clearpage
\input{content/6_appendex}

\end{document}

%% file: content/0_abstract.tex
\begin{abstract}
As supervised fine-tuning (SFT) becomes central to post-training large language models, target-only training can yield diminishing returns while hard samples remain under-learned. While existing methods continue to reinforce target predictions, informative corrective signal can remain latent in models' own historical weak states.
Motivated by this observation, we study \emph{Weak-Driven Learning}, a post-training direction that repurposes weak historical checkpoints as corrective signal for continued optimization. We instantiate this direction with \modelname{} (\textbf{W}eak Agents Can \textbf{M}ake \textbf{S}trong Agents \textbf{S}tronger), a training framework that combines \textbf{W}eak-\textbf{D}riven \textbf{D}iscrepancy \textbf{S}election (WD-DS) with \textbf{W}eak-\textbf{D}riven \textbf{J}oint \textbf{T}raining (WD-JT). By identifying recoverable learning gaps via entropy dynamics and reinforcing them through weak-driven logit mixing, \modelname{} provides one practical way to reuse historical weak states during training.
Experiments on mathematical reasoning, code generation, and logical reasoning across four 3B--8B base models (Qwen3-4B/8B, Qwen2.5-3B, Gemma-3-4B-PT) show gains over competitive post-training baselines, while incurring zero additional deployment-time inference cost.
\end{abstract}

%% file: content/1_introduction.tex
\section{Introduction}

Supervised Fine-Tuning (SFT)~\cite{ouyang2022training,touvron2023llama,zou2025transformer,chen2025llmboost} has become the standard post-training paradigm for improving large language models, translating pretrained capability into usable reasoning and instruction-following performance. Meanwhile, Knowledge Distillation (KD)~\cite{hinton2015distilling,Gou_2021,agarwal2023gkd} has emerged as a particularly effective approach for training small and medium-sized LLMs by matching the outputs or trajectories of stronger models. These imitation-based strategies are powerful when high-quality supervision is available, but such supervision is often tied to proprietary, expensive, or domain-specific teacher models that are not always accessible. This motivates post-training methods that can obtain useful learning signals without relying on an external stronger teacher.

A growing line of self-improvement methods follows this direction by constructing supervision from the model itself. However, in most such methods, the model-derived signal must first be converted into explicit supervision through additional computation, such as sampling self-generated responses for preference construction~\citep{chen2024self} or running auxiliary contexts to form teacher targets~\citep{mitra2025semanticsoft}. Rather than generating new supervision, we look back at a source already created by ordinary training: the weak agents produced along the training trajectory. The intuition comes from human collaboration: a strong problem solver can improve by observing and correcting the mistakes of a weaker collaborator, because those mistakes reveal concrete errors that must be ruled out. We therefore ask: \textbf{\textit{Can the weak agents already produced during standard training be reused directly as corrective signal?}}

We answer this question by exploring the reverse direction of knowledge distillation and challenging the assumption that weaker models are useful only as students. We study this idea as \emph{Weak-Driven Learning} and instantiate it as \modelname{} (\textbf{W}eak agents \textbf{M}ake \textbf{S}trong agents \textbf{S}tronger). As illustrated in Figure~\ref{fig:drivenlearning}, weak agents are not useful because their entire output distributions should be imitated. On routine tokens, weak and strong agents may behave similarly and provide little additional signal. Their value appears at key decision tokens, where the weak agent often remains uncertain or assigns mass to weak-revealed hard negatives: incorrect non-target tokens that mark unresolved boundaries. These token-level weak--strong discrepancies expose where the strong agent still needs to sharpen its decision boundary. \modelname{} therefore does not imitate the weak distribution; instead, it injects weak logits into a supervised mixed-logit objective where the ground-truth token remains the target, turning weak-agent mistakes into token-level corrective signal. Thus, otherwise discarded weak agents can become reusable corrective signal during training without deployment-time overhead.

\begin{figure}[t]
    \centering
    \includegraphics[width=0.94\linewidth]{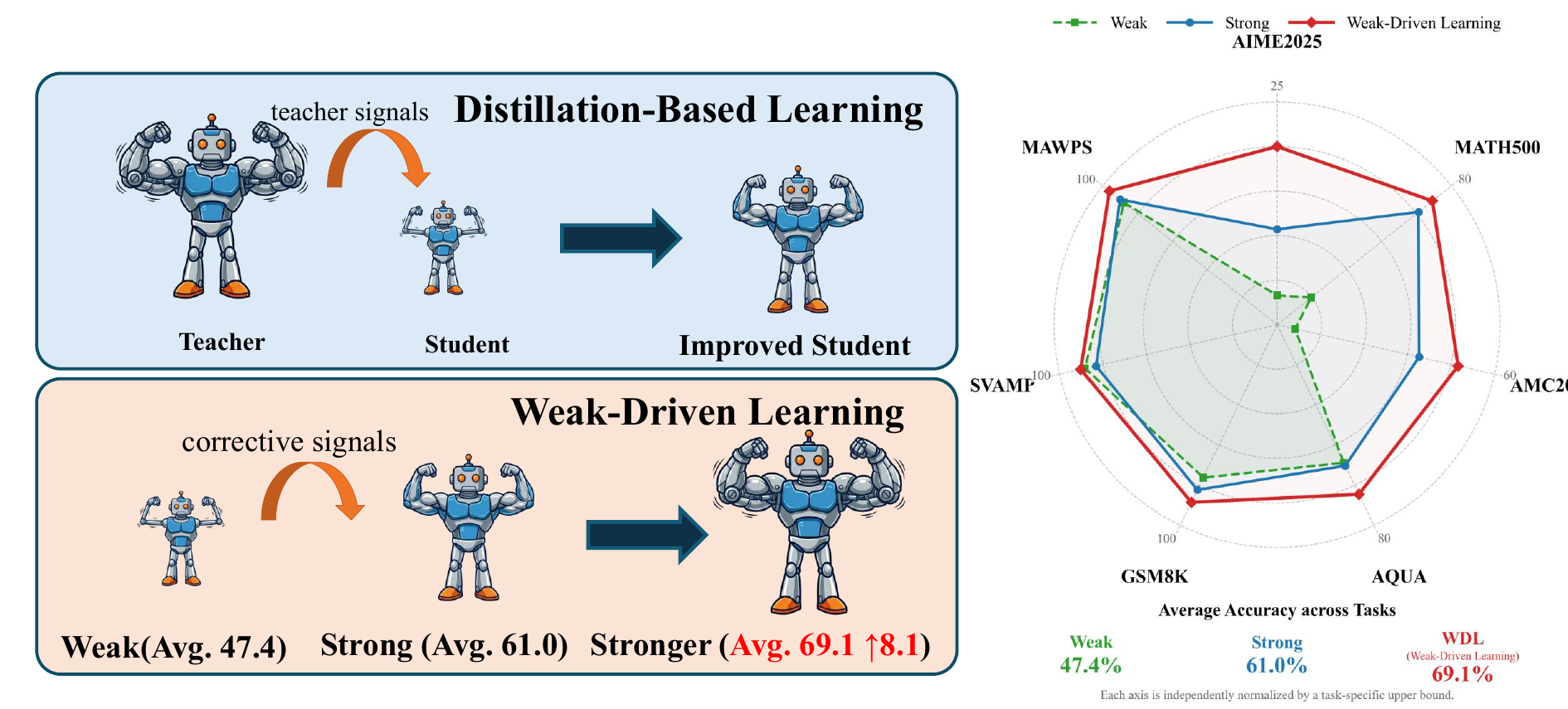}
    \caption{\textbf{Paradigm Comparison:} Distillation-Based Learning vs. Weak-Driven Learning.}
    \label{fig:drivenlearning}
\end{figure}

Our concrete implementation is \modelname{}, a practical post-training framework built around a core mixed-logit mechanism and an auxiliary discrepancy-selection module. \textbf{Weak-Driven Joint-Logit Training (WD-JT)} forwards the weak and strong agents on the same batch, fuses their logits, and computes cross-entropy against the ground-truth token to optimize the training pair while retaining the strong branch as the output model. Because weak logits assign probability mass to weak-revealed hard negatives, this objective turns weak-model errors into token-level corrective signal for the strong branch. \textbf{Weak-Driven Discrepancy Selection (WD-DS)} then improves the efficiency of WD-JT by selecting where this weak--strong computation should be spent: it forwards both checkpoints on the training set, computes token-averaged entropy over supervised completion tokens, and resamples an active distribution using weak-model difficulty and weak--strong entropy changes. In this way, WD-DS prioritizes samples that are historically hard, brittle, or forgotten by the current strong model, where weak-model confusion is more likely to provide useful corrective signal for WD-JT. We further formalize weak-driven learning on hard samples and provide a conditional gradient-level analysis of the mixed-logit mechanism in this regime.

More broadly, our findings suggest that continued post-training need not rely exclusively on stronger external teachers or newly curated annotations: historical training states can serve as reusable supervision. Our contributions are summarized as follows:
\begin{itemize}
    \setlength{\itemsep}{0.2em}
    \setlength{\topsep}{0em}
    \setlength{\partopsep}{0em}
    \setlength{\parsep}{0em}
    \item \textbf{Weak-driven post-training.} We introduce \emph{Weak-Driven Learning}, a new post-training direction in which weak historical checkpoints are used as corrective signal for strengthening a model on hard samples.
    \item \textbf{Training framework.} We instantiate this direction with \modelname{}, a practical post-training framework that combines \textbf{WD-DS} for entropy-discrepancy sample selection and \textbf{WD-JT} for weak-driven logit mixing. Together, these two modules operationalize weak-driven learning without additional inference overhead.
    \item \textbf{Theoretical analysis.} We provide a conditional gradient-level analysis of the mixed-logit mechanism, showing how incorporating weak-model logits can increase negative-token gradients under explicit margin assumptions.
    \item \textbf{Empirical performance.} Using SFT as the anchor, \modelname{} improves Math-Avg by up to $+5.9$ points and Code-Avg by up to $+4.4$ points across the reported backbones. Compared with the strongest baseline, \modelname{} nearly doubles the average gain over SFT on the main math/code settings, while incurring no additional deployment-time inference cost.
\end{itemize}

%% file: content/2_related_work.tex
\section{Related Work}
\label{sec:related_work}

\paragraph{Post-training: distillation and perturbation-based regularization.}
Post-training commonly starts from supervised fine-tuning (SFT)~\cite{ouyang2022training,touvron2023llama,yang2026grouprelativeadvantagebiased} and often uses knowledge distillation from a larger teacher~\cite{hinton2015distilling,agarwal2023gkd,gu2023minillm} or smoothed self-predictions~\cite{zhang2019teacherimproveperformanceconvolutional,xu2023wizardlm}. Such objectives transfer knowledge effectively, but can leave hard or high-discrepancy samples insufficiently corrected once aggregate SFT gains saturate. Perturbation regularizers such as \textbf{NEFTune}~\cite{jain2023neftunenoisyembeddingsimprove} instead inject noise for regularization, rather than using the model's historical error patterns as corrective signals.

\paragraph{Self-improvement and weak-to-strong supervision.}
Another line of work improves models using weaker or alternate-context variants, including weak-to-strong generalisation~\cite{burns2023weak,gulcehre2023reinforcedselftrainingrestlanguage,somerstep2025transfer}, weak-supervision reasoning~\cite{yang2024weaktostrongreasoning,yuan2026incentivizing}, preference optimisation~\cite{zhu2024weakstrongdpo}, and self-improvement methods such as \textbf{SPIN}~\cite{chen2024self}, \textbf{SSB}~\cite{mitra2025semanticsoft}, and \textbf{LightReasoner}~\cite{wang2025lightreasoner}. Nearby work also studies zero-to-strong elicitation~\cite{liu2025zerotostrong}, weak-to-strong test-time search~\cite{zhou2024weakstrongsearch}, tool generalisation~\cite{he2025gentool}, adaptive curricula~\cite{chen2025selfevolvingcurriculum}, differential-entropy data selection~\cite{su2026instructdiff}, and capability drift~\cite{thede2026captrack}.

\paragraph{Difference from prior post-training methods.}
The above methods often rely on inference-time steering, rollout curricula, or selection-only entropy scores; in contrast, \modelname{} reuses historical model states as corrective signals during post-training and avoids deployment-time search.

%% file: content/3_method.tex
\section{Weak-Driven Learning}
\label{sec:preliminaries}

\subsection{Preliminaries}
We consider an autoregressive language model $\mathcal{M}_{\theta}$ that maps a context sequence $x$ to logits $z_t \in \mathbb{R}^{|\mathcal{V}|}$ over a vocabulary $\mathcal{V}$, inducing $P_{\theta}(\cdot \mid x)=\mathrm{Softmax}(z_t)$.
\textbf{Supervised fine-tuning.}
In standard supervised fine-tuning (SFT), we minimize the negative log-likelihood over a dataset $\mathcal{D}$:
\begin{equation}
\label{eq:sft}
\mathcal{L}_{\text{SFT}}(\theta)=-\mathbb{E}_{(x,y)\sim \mathcal{D}}\!\left[\log P_{\theta}(y \mid x)\right],
\end{equation}
where $y$ denotes the ground-truth next token under context $x$ (we omit the time index when clear).\\
\textbf{Predictive entropy.}
To quantify uncertainty and weak--strong discrepancy, we monitor the predictive entropy $H(P_{\theta}(\cdot \mid x))=-\sum_{v\in\mathcal{V}}P_{\theta}(v \mid x)\log P_{\theta}(v \mid x)$.
Let $\ell(x,y;\theta)=-\log P_{\theta}(y\mid x)$ be the per-token cross-entropy loss. Its gradient w.r.t.\ the logit of any token $k\in\mathcal{V}$ is $\partial \ell/\partial z_t[k]=P_{\theta}(k \mid x)-\mathbb{I}[k=y]$. In particular, for any negative token $k\neq y$, $|\partial \ell/\partial z_t[k]|=P_{\theta}(k \mid x)$, i.e., the gradient magnitude on a negative class is proportional to its assigned probability.

\subsection{Problem: Weak-Driven Learning}
\label{subsec:weak_driven_learning}
\label{subsec:weak_driven_hard_samples}

We define \emph{Weak-Driven Learning} as a class of post-training paradigms in which the improvement of a strong agent is driven by its systematic discrepancies with a weaker agent, rather than by imitating a stronger teacher or directly imitating the weak agent. Formally, let $\mathcal{D}=\{(x_i,y_i)\}_{i=1}^N$ be a post-training dataset, let $\mathcal{M}_{\mathrm{strong}}$ denote the current strong agent, and let $\mathcal{M}_{\mathrm{weak}}$ denote a weaker agent, such as a historical checkpoint from the same training trajectory or a related lower-capability model. For an input $x$, let
\[
    o_{\mathrm{strong}}(x)
    =
    \mathcal{M}_{\mathrm{strong}}(x),
    \qquad
    o_{\mathrm{weak}}(x)
    =
    \mathcal{M}_{\mathrm{weak}}(x)
\]
denote their outputs, which may take the form of logits, probabilities, hidden representations, generated responses, rewards, uncertainty estimates, or other model-derived signals.

The goal of weak-driven learning is to construct a discrepancy-aware training signal
\begin{equation}
    \mathcal{S}_{\mathrm{WD}}
    =
    \Phi\!\left(
        o_{\mathrm{strong}}(x),
        o_{\mathrm{weak}}(x),
        y
    \right),
    \qquad (x,y)\in\mathcal{D},
    \label{eq:wd_signal_general}
\end{equation}
where $\Phi$ transforms weak--strong discrepancies into corrective supervision for the strong agent while preserving the task target $y$ as the optimization anchor. 
The corresponding post-training objective can be written abstractly as
\begin{equation}
    \left(
    \mathcal{M}_{\mathrm{strong}}^{+},
    \mathcal{M}_{\mathrm{weak}}
    \right)
    =
    \arg\min_{\mathcal{M}_{\mathrm{strong}},\,\mathcal{M}_{\mathrm{weak}}}
    \sum_{(x,y)\in\mathcal{D}}
    \mathcal{L}_{\mathrm{WD}}
    \left(
        \mathcal{M}_{\mathrm{strong}},
        \mathcal{M}_{\mathrm{weak}};
        \mathcal{S}_{\mathrm{WD}}(x,y)
    \right),
    \label{eq:wd_objective_general}
\end{equation}
where both the strong and weak agents are optimized during post-training, but \textbf{only the updated strong agent $\mathcal{M}_{\mathrm{strong}}^{+}$ is retained for downstream use}.

In this problem, the weak agent is not treated as a teacher to imitate. Instead, its errors, uncertainty, and disagreement with the strong agent reveal residual weaknesses, unstable predictions, or under-optimized regions of the strong agent. Thus, weak-driven learning converts weaker agents from discarded training artifacts into useful sources of corrective post-training signal.

\section{Method: WMSS}
\label{sec:methods}
\label{sec:methodology}

\modelname{} is our practical instantiation of weak-driven learning. It consists of two modules: \textbf{Weak-Driven Joint-Logit Training (WD-JT)}, the core mechanism that converts weak-model confusion into gradient signal through logit mixing, and \textbf{Weak-Driven Discrepancy Selection (WD-DS)}, an auxiliary module that uses weak--strong entropy discrepancies to identify hard or high-discrepancy samples for continued learning.

In this section, we describe how \modelname{} operationalizes weak-driven learning to further strengthen a strong agent.
\paragraph{Weak and strong agent initialization.}
In \modelname, the weak agent $\mathcal{M}_{\mathrm{weak}}$ is initialized from a historical checkpoint and provides weak logits during WD-JT. We begin with \textbf{Phase~1 (Initialization)}: starting from a base model $\mathcal{M}_0$, we perform standard SFT to obtain $\mathcal{M}_1$, and set
\begin{equation}
\label{eq:wmss_init}
\mathcal{M}_{\mathrm{weak}} \leftarrow \mathcal{M}_0,\qquad
\mathcal{M}_{\mathrm{strong}} \leftarrow \mathcal{M}_1.
\end{equation}
The weak agent provides a corrective signal via its logits $z_{\mathrm{weak}}(x) \in \mathbb{R}^{|\mathcal V|}$, which preserve a softer decision boundary and highlight weak-revealed hard negatives, stabilising continued optimisation. For the iterative variant ($K{>}1$ in Algorithm~\ref{alg:training_pipeline}), the weak agent for each outer iteration is initialized from the previous iteration's output checkpoint.

\begin{figure}[t]
    \centering
    \includegraphics[width=\textwidth]{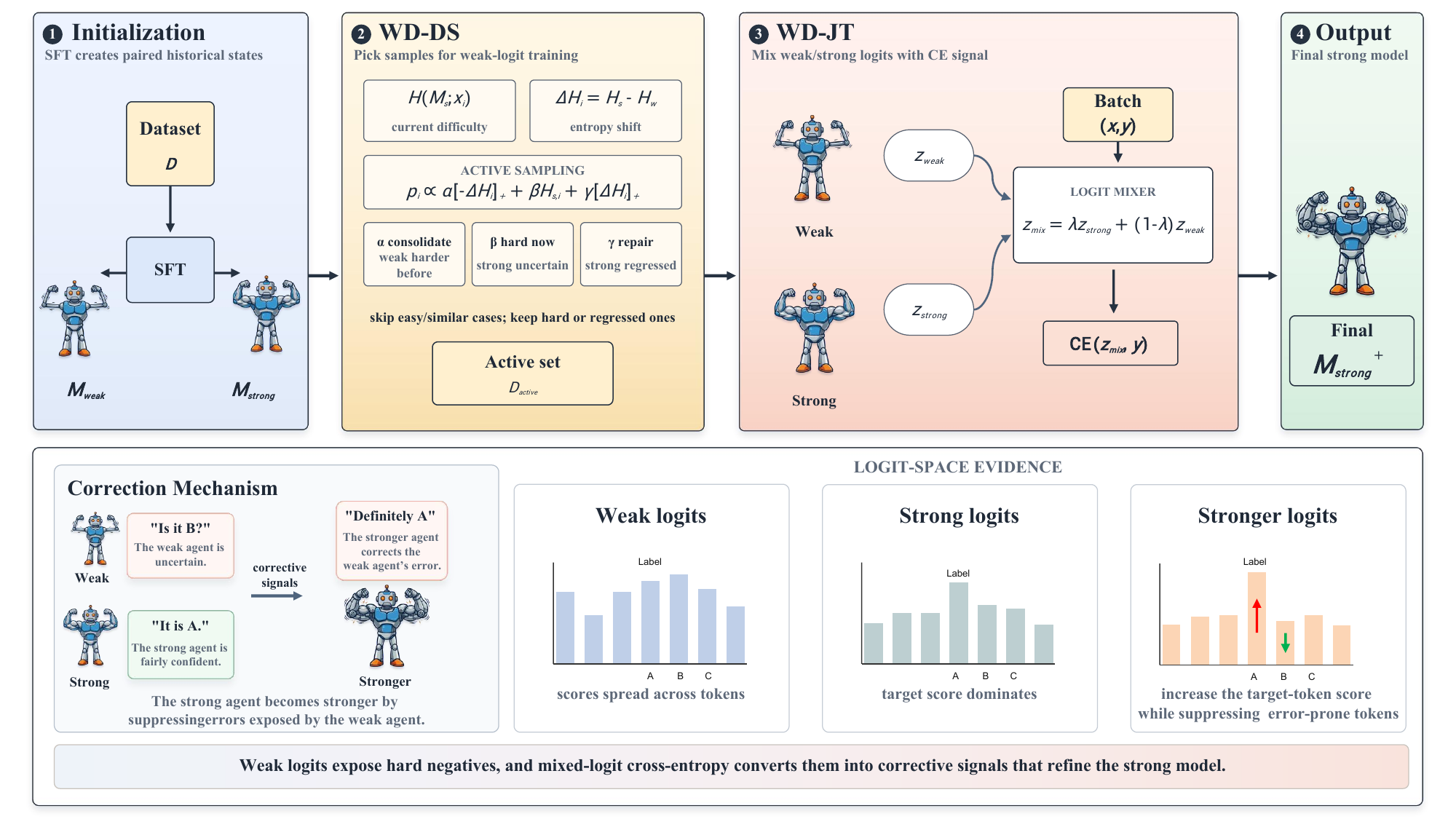}
    \caption{Weak-Driven Learning instantiated by \modelname. Overview of the \modelname{} framework. It has three phases: (1) initialization, (2) Weak-Driven Discrepancy Selection (WD-DS), and (3) Weak-Driven Joint-Logit Training (WD-JT); the right panel visualizes the weak-driven mixed-logit principle through logit mixing and gradient amplification.}
    \label{fig:pivot}
\end{figure}

\subsection{WD-JT: Weak-Driven Joint-Logit Training}
\label{subsec:joint_training}

The core weak-driven operation in \modelname{} is WD-JT: given a weak agent $\mathcal{M}_{\mathrm{weak}}$ and a strong agent $\mathcal{M}_{\mathrm{strong}}$, we optimize the paired agents through a weak-conditioned mixed-logit objective during training, converting weak-model confusion into additional optimization signal. The strong branch is retained as the final improved model. This module is the primary mechanism by which \modelname{} instantiates weak-driven learning.

\paragraph{Weak-conditioned logit mixing.}
For a training pair $(x,y)$, let $z_{\mathrm{strong}}(x), z_{\mathrm{weak}}(x)\in\mathbb{R}^{|\mathcal V|}$ denote the next-token logits produced by $\mathcal{M}_{\mathrm{strong}}$ and $\mathcal{M}_{\mathrm{weak}}$, respectively. We construct a mixed probability map by linearly mixing the logits:
\begin{equation}
\label{eq:logit_mix}
z_{\mathrm{mix}}(x)
=
\lambda\, z_{\mathrm{strong}}(x)
+
(1-\lambda)\, z_{\mathrm{weak}}(x),
\end{equation}
where $ \lambda\in[0,1]$.
We then optimize the training pair using the mixed-logit distribution
$P_{\mathrm{mix}}(\cdot\mid x) := \mathrm{Softmax}(z_{\mathrm{mix}}(x))$:
\begin{equation}
\label{eq:joint_sft_loss}
\mathcal{L}_{\mathrm{mix}}
=
- \mathbb{E}_{(x,y)\sim\mathcal D}
\left[
\log P_{\mathrm{mix}}(y\mid x)
\right]
\end{equation}

This mixed-logit objective couples the weak agent's predictive distribution with the paired training update.
Rather than treating the weak agent as a teacher, WD-JT uses its output distribution as token-level corrective signal.
Because the weak agent retains relatively higher probability on weak-revealed hard negatives than the strong agent, the fused distribution exposes non-target tokens that standard SFT may have driven toward negligible probability.
Computing ground-truth cross-entropy on this fused distribution converts these hard negatives into corrective signal for hard samples.

\subsection{WD-DS: Weak-Driven Discrepancy Selection}
\label{subsec:discrepancy_selection}

WD-DS is a lightweight selector for WD-JT. While WD-JT specifies \emph{how} weak agents provide corrective signal through logit mixing, WD-DS decides \emph{where} this weak--strong computation is most useful. It prioritizes samples that remain learnable and expose weak--strong discrepancies, and downplays samples already confidently handled by both agents or samples where the two models behave too similarly to reveal weak-model mistakes.

\paragraph{Entropy-based selection.}
For each supervised sequence example $x_i=(u_i,y_i)$, let $\mathcal{T}_i$ denote the target-token positions included in the training loss and let $c_{i,t}=(u_i,y_{i,<t})$ be the prompt concatenated with the ground-truth target prefix before position $t$. We use the length-normalized predictive entropy
\begin{equation}
H(\mathcal{M};x_i)
=
-\frac{1}{|\mathcal{T}_i|}
\sum_{t\in\mathcal{T}_i}
\sum_{v\in\mathcal{V}}
P_{\mathcal{M}}(v\mid c_{i,t})
\log P_{\mathcal{M}}(v\mid c_{i,t}).
\end{equation}
The weak--strong discrepancy is
\begin{equation}
\Delta H_i
=
H(\mathcal{M}_{\mathrm{strong}};x_i)
-
H(\mathcal{M}_{\mathrm{weak}};x_i).
\end{equation}
WD-DS samples examples according to
\begin{equation}
\label{eq:sampling_prob}
p_i \propto
\alpha[-\Delta H_i]_+
+
\beta H(\mathcal{M}_{\mathrm{strong}};x_i)
+
\gamma[\Delta H_i]_+,
\end{equation}
where $[u]_+ := \max(u,0)$ and $p_i$ is normalized over $\mathcal{D}$. The three terms select consolidation samples, currently difficult samples, and recoverable regressions. Given $p_i$, WD-DS draws $|\mathcal{D}|$ examples with replacement and removes duplicate sampled indices to form $\mathcal{D}_{\mathrm{active}}$; WD-JT then trains on $\mathcal{D}_{\mathrm{active}}$ with the same mixed-logit objective in Eq.~\eqref{eq:joint_sft_loss}.

\subsection{Training Pipeline}
\label{subsec:pipeline}

Algorithm~\ref{alg:training_pipeline} summarizes \modelname{} and the roles of its two modules.\\
\textbf{Phase 1 (Initialization)} trains $\mathcal{M}_0$ with SFT to obtain $\mathcal{M}_{\mathrm{strong}}$ and initializes $\mathcal{M}_{\mathrm{weak}} \leftarrow \mathcal{M}_0, \mathcal{M}_{\mathrm{strong}} \leftarrow \mathcal{M}_1$.\\
\textbf{Phase 2 (WD-DS)} constructs an active training distribution by computing token-averaged sequence entropy between $\mathcal{M}_{\mathrm{weak}}$ and the current $\mathcal{M}_{\mathrm{strong}}$, yielding a selected active dataset $\mathcal{D}_{\mathrm{active}}$ via Eq.~\eqref{eq:sampling_prob}. \\
\textbf{Phase 3 (WD-JT)} then performs weak-driven mixed-logit training on $\mathcal{D}_{\mathrm{active}}$ by mixing logits as in Eq.~\eqref{eq:logit_mix} and optimizing the mixed-logit cross-entropy loss in Eq.~\eqref{eq:joint_sft_loss}. This allows the training pair to learn from weak-revealed hard negatives on hard samples, producing $\mathcal{M}_{\mathrm{strong}}^{+}$ from the retained strong branch.

\begin{algorithm}[t]
\begin{minipage}{0.7\textwidth}
  \caption{Weak Agents make Strong Agents Stronger}
  \label{alg:training_pipeline}
  \begin{algorithmic}[1]
    \Require Dataset $\mathcal{D}$, Base Model $\mathcal{M}_0$
    \Require Iterations $K$, Hyperparams $\alpha,\beta,\gamma,\lambda$
    \State \textbf{Phase 1: Initialization}
    \State $\mathcal{M}_1 \leftarrow \text{Train}(\mathcal{M}_0, \mathcal{D}, \mathcal{L}_{\text{SFT}})$
    \State \textbf{Phase 2: Iterative Training Loop}
    \For{$t = 1$ to $K$}
        \State \textit{// Step A: Weak-Driven Discrepancy Selection}
        \State Compute token-averaged sequence entropies as defined above
        \State $\Delta \mathbf{H} \leftarrow \mathbf{H}(\mathcal{M}_t) - \mathbf{H}(\mathcal{M}_{t-1})$
        \State Calculate $\mathbf{p}$ via Eq.~\ref{eq:sampling_prob}
        \State Sample and deduplicate to form $\mathcal{D}_{\text{active}}$

        \State \textit{// Step B: Weak-Driven Learning}
        \State Initialize $(\mathcal{M}_{\mathrm{weak}}^{(t)}, \mathcal{M}_{\mathrm{strong}}^{(t)}) \leftarrow (\mathcal{M}_{t-1}, \mathcal{M}_{t})$
        \For{batch $(x, y) \in \mathcal{D}_{\text{active}}$}
            \State $z_{\text{weak}} \leftarrow \text{Forward}(\mathcal{M}_{\mathrm{weak}}^{(t)}, x)$
            \State $z_{\text{strong}} \leftarrow \text{Forward}(\mathcal{M}_{\mathrm{strong}}^{(t)}, x)$
            \State $z_{\text{mix}} \leftarrow \lambda z_{\text{strong}} + (1-\lambda)z_{\text{weak}}$
            \State Update $(\mathcal{M}_{\mathrm{weak}}^{(t)}, \mathcal{M}_{\mathrm{strong}}^{(t)})$ on $\mathrm{CE}(z_{\text{mix}}, y)$
        \EndFor
        \State $\mathcal{M}_{t+1} \leftarrow \mathcal{M}_{\mathrm{strong}}^{(t)}$
    \EndFor
    \State \textbf{return} $\mathcal{M}_{K+1}$
  \end{algorithmic}
\end{minipage}
\end{algorithm}

\section{Mechanistic Analysis of Weak-Driven Learning}
\label{sec:theoretical_analysis}

We next provide a conditional gradient-level explanation for why weak-driven logit mixing can help: \textbf{mixing strong logits with a weak agent can amplify gradient signal on weak-revealed hard negatives} under the margin conditions stated below. The analysis is local and mechanism-oriented: it explains one sufficient regime in which WD-JT increases probability mass on non-target tokens in $\mathcal{H}$, thereby converting weak-agent confusion into corrective signal for the hard-sample regime targeted by WD-DS.

\paragraph{Negative-token gradient setup.}
Let $\ell(x,y;\theta) = -\log P_{\theta}(y\mid x)$. Its gradient with respect to any logit is $\partial \ell / \partial z_t[k] = P_{\theta}(k\mid x) - \mathbb{I}[k=y]$, so for every negative token $k\neq y$ the gradient magnitude equals $P_{\theta}(k\mid x)$. On hard samples, target-only SFT can assign very small probability to weak-revealed hard negatives before they are fully resolved, leaving limited gradient signal on the corresponding negative tokens. \modelname{} exploits this by increasing probability mass on these hard negatives through logit mixing.

Consider a weak agent $\mathcal{M}_{\mathrm{weak}}$ initialized from a historical checkpoint and a strong agent $\mathcal{M}_{\mathrm{strong}}$ that, for a context $x$, produce logits $z_{\mathrm{weak}}(x), z_{\mathrm{strong}}(x) \in \mathbb{R}^{|\mathcal{V}|}$, and the mixed logits
\begin{equation}
    z_{\text{mix}}(x) = (1-\lambda) z_{\mathrm{weak}}(x) + \lambda z_{\mathrm{strong}}(x), \quad \lambda \in [0,1].
    \label{eq:main_mix}
\end{equation}
Operationally, $z_{\text{mix}}(x)$ is the fused logit map used for the paired training update, injecting the weak agent's uncertainty while preserving the strong agent's target direction.
Let $y$ denote the target index, $e_y$ the one-hot vector, and define $P_{\mathrm{mix}}(\cdot\mid x)=\mathrm{Softmax}(z_{\text{mix}}(x))$ and $P_{\mathrm{strong}}(\cdot\mid x)=\mathrm{Softmax}(z_{\mathrm{strong}}(x))$. More generally, for any logit map $z(x)$, let $P_z(\cdot\mid x)=\mathrm{Softmax}(z(x))$. The cross-entropy gradient on fused logits is
\begin{equation}
    g_{\text{mix}} = \nabla_{z_{\text{mix}}(x)}\mathcal{L} = P_{\mathrm{mix}}(\cdot\mid x) - e_y,
    \label{eq:main_grad_mix}
\end{equation}
For any negative token $k \neq y$, $g_{\text{mix}}[k] = P_{\mathrm{mix}}(k\mid x)$, so increasing negative probability mass directly increases the gradient magnitude on that token. Standard SFT on the strong agent gives $g_{\text{sft}} = P_{\mathrm{strong}}(\cdot\mid x) - e_y$ and provides little signal on a hard negative once $P_{\mathrm{strong}}(k\mid x)$ is already small.

\textbf{Margins and hard negatives.} Define the target margin for any negative token $k \neq y$ as $m_k(z(x)) = z(x)[y] - z(x)[k]$. Smaller margins mean higher confusion. Define the hard-negative set
\begin{equation}
    \mathcal{H} = \{ k \neq y : m_k(z_{\mathrm{weak}}(x)) < m_k(z_{\mathrm{strong}}(x)) \},
    \label{eq:main_hard_set}
\end{equation}
i.e., tokens where the weak agent is less separated than the strong agent, highlighting unresolved boundaries. The mixed margin is a convex combination $m_k(z_{\text{mix}}(x)) = (1-\lambda)m_k(z_{\mathrm{weak}}(x)) + \lambda m_k(z_{\mathrm{strong}}(x))$, so mixing shrinks margins toward the weak agent on $\mathcal{H}$ and raises the relative probability of hard negatives.

\begin{theorem}[Total negative-mass increase under uniform margin shrinkage]
If $m_k(z_{\mathrm{weak}}(x)) \le m_k(z_{\mathrm{strong}}(x))$ for all $k \neq y$, then
\begin{align}
    P_{\mathrm{mix}}(y\mid x) &\le P_{\mathrm{strong}}(y\mid x), \nonumber\\
    \sum_{k \neq y} P_{\mathrm{mix}}(k\mid x) &\ge \sum_{k \neq y} P_{\mathrm{strong}}(k\mid x).
    \label{eq:main_total_neg_mass}
\end{align}
\end{theorem}
This sufficient condition shows that, when the weak agent is uniformly more uncertain than the strong agent, mixed-logit training shifts probability mass from the target to negative tokens. In that regime, gradient magnitude on negative classes can be amplified on hard samples.

\begin{corollary}[Per-token amplification on hard negatives]
For any $k \in \mathcal{H}$, the mixed negative gradient satisfies
\begin{align}
    P_{\mathrm{mix}}(k\mid x) &\ge P_{\mathrm{strong}}(k\mid x) \quad \text{whenever} \nonumber\\
    \frac{P_{\mathrm{mix}}(y\mid x)}{P_{\mathrm{strong}}(y\mid x)}
    &\ge \exp\!\left(-(1-\lambda)\Delta m_k\right),
    \label{eq:main_amp_condition}
\end{align}
where $\Delta m_k = m_k(z_{\mathrm{strong}}(x))-m_k(z_{\mathrm{weak}}(x))>0$.
\end{corollary}
For hard negatives, mixed-logit training increases their assigned probability---and thus their gradient magnitude---under a mild, explicit condition on the relative target probabilities.

\begin{proposition}[Branch-level logit updates on negative and target tokens]
Under a local diagonal-kernel approximation to the mixed-logit paired update in Algorithm~\ref{alg:training_pipeline}, for branch $i\in\{\mathrm{weak},\mathrm{strong}\}$ with $s_{\mathrm{weak}}=1-\lambda$ and $s_{\mathrm{strong}}=\lambda$,
\begin{align}
    \Delta z_{i,k} &\approx -\eta s_i\, P_{\mathrm{mix}}(k\mid x) \quad (k \neq y), \nonumber\\
    \Delta z_{i,y} &\approx \eta s_i\, (1-P_{\mathrm{mix}}(y\mid x)).
    \label{eq:main_update}
\end{align}
\end{proposition}
Any increase in $P_{\mathrm{mix}}(k\mid x)$ amplifies the local gradient signal on hard negatives in the paired objective, while a decrease in $P_{\mathrm{mix}}(y\mid x)$ strengthens the upward push on the target logit. The full parameter-space update passes through the Jacobian Gram matrix and is derived in Appendix~\ref{sec:appendix_joint_theory}; the simplified form above states the token-level direction of the cross-entropy signal. These dynamics are consistent with the empirical reduction of the non-target logit mean observed in our logit statistics.

\input{table/main_results}

\paragraph{Mechanistic interpretation.}
The result explains the ``why'' of WD-JT: weak logits can expose weak-revealed hard negatives inside a supervised CE objective, and the resulting larger negative-token probabilities produce stronger corrective updates against those tokens. Appendix~\ref{sec:appendix_joint_theory} provides additional dynamics diagnostics for branch sensitivity, confident-regime shielding, and mean-logit drift, while Appendix~\ref{app:theory_discussion} discusses the validity of the margin conditions and multi-step dynamics.

%% file: table/main_results.tex
\begin{table*}[t!]
    \centering
    \small
    \renewcommand{\arraystretch}{0.96}
    \caption{\textbf{Main results.} pass@1 accuracy (\%) under greedy decoding. \textbf{All methods are trained for two epochs}; \modelname{} uses the same two-epoch total budget. Additional results on different domains (logic) and model families (Gemma) are reported in Appendix~\ref{app:logic_results} and Appendix~\ref{app:more_models}.}
    \label{tab:main_results}
    \resizebox{\textwidth}{!}{
    \begin{tabular}{@{}l|c c c c c c c c|c c c c c}
        \toprule
        \multirow{2}{*}{\textbf{Methods}} & \multicolumn{8}{c|}{\textbf{Math}} &  \multicolumn{5}{c}{\textbf{Code}} \\
        \cmidrule(l){2-9}\cmidrule(l){10-14}
         &\textbf{AIME2025} & \textbf{MATH500} & \textbf{AMC23} & \textbf{AQUA} & \textbf{GSM8K} & \textbf{MAWPS} & \textbf{SVAMP}& Avg. & \textbf{HE+} & \textbf{MBPP+} & \textbf{BCB} & \textbf{LCB} & Avg. \\
        \midrule
        \multicolumn{14}{c}{\textit{Qwen3-4B-Base as base model.}} \\
        \midrule
         SFT & 13.3 & 66.2 & \underline{47.5} & 61.8 & 83.9 & 91.2 & 85.5 & 64.2 & 73.8 & 64.0 & 38.5 & 17.8 & 48.5\\
         UNDIAL & 13.3 & 63.2 & 40.0 & 55.1 & 82.9 & 90.8 & 84.1 & 61.3 \textcolor{green!50!black}{$\downarrow$2.9} & 72.6 & 66.8 & 39.6 & 17.5 & 49.1 \textcolor{red}{$\uparrow$0.6}\\
         NEFTune & 16.7 & 68.2 & 42.5 & 59.1 & 86.7 & 93.7 & 87.7 & 64.9 \textcolor{red}{$\uparrow$0.7} & 75.0 & 66.9 & 40.3 & 16.0 & 49.6 \textcolor{red}{$\uparrow$1.1}\\
         SPIN & \underline{20.0} & \underline{70.8} & \underline{47.5} & \underline{66.9} & 85.8 & 87.4 & 87.6 & \underline{66.6 \textcolor{red}{$\uparrow$2.4}} & 76.2 & 66.7 & 39.1 & \underline{18.8} & 50.2 \textcolor{red}{$\uparrow$1.7}\\
         SSB & 10.0 & 70.6 & 45.0 & 66.1 & \underline{87.1} & \underline{95.4} & \underline{88.8} & 66.1 \textcolor{red}{$\uparrow$1.9} & \underline{78.0} & \underline{67.4} & \underline{42.5} & 18.0 & \underline{51.5 \textcolor{red}{$\uparrow$3.0}}\\
         \textbf{\modelname} & \textbf{20.0} & \textbf{71.4} & \textbf{50.0} & \textbf{67.7} & \textbf{88.5} & \textbf{96.2} & \textbf{90.3} & \textbf{69.1 \textcolor{red}{$\uparrow$4.9}} & \textbf{78.7} & \textbf{68.2} & \textbf{43.5} & \textbf{21.0} & \textbf{52.9 \textcolor{red}{$\uparrow$4.4}}\\
        \midrule
        \multicolumn{14}{c}{\textit{Qwen3-8B-Base as base model.}} \\
        \midrule
         SFT & 16.7 & 72.2 & 45.0 & 63.0 & 87.5 & 95.0 & 88.6 & 66.9 & 80.5 & 69.2 & 44.5 & 21.0 & 53.8\\
         UNDIAL & 10.0 & 71.6 & \underline{47.5} & 66.9 & 89.9 & 96.6 & 91.2 & 67.6 \textcolor{red}{$\uparrow$0.7} & 79.9 & 69.4 & 45.7 & 21.5 & 54.1 \textcolor{red}{$\uparrow$0.3}\\
         NEFTune & 16.7 & 73.8 & 40.0 & 71.3 & \underline{90.3} & \underline{97.5} & 91.5 & 68.7 \textcolor{red}{$\uparrow$1.8} & 81.1 & \underline{70.4} & 45.9 & 21.0 & 54.6 \textcolor{red}{$\uparrow$0.8}\\
         SPIN & 16.7 & 72.4 & \underline{47.5} & \textbf{78.3} & 88.8 & 95.4 & 92.7 & 70.3 \textcolor{red}{$\uparrow$3.4} & \underline{81.7} & 70.0 & \underline{46.3} & \underline{21.8} & \underline{55.0 \textcolor{red}{$\uparrow$1.2}}\\
         SSB & 16.7 & \underline{74.6} & \underline{52.5} & 76.4 & 89.3 & 93.3 & \underline{93.0} & \underline{70.8 \textcolor{red}{$\uparrow$3.9}} & 81.1 & 66.9 & 46.0 & 20.8 & 53.7 \textcolor{green!50!black}{$\downarrow$0.1}\\
         \textbf{\modelname} & \textbf{20.0} & \textbf{75.4} & \textbf{52.5} & \underline{77.2} & \textbf{92.5} & \textbf{97.9} & \textbf{94.0} &  \textbf{72.8 \textcolor{red}{$\uparrow$5.9}} & \textbf{84.1} & \textbf{71.4} & \textbf{47.1} & \textbf{22.3} & \textbf{56.2 \textcolor{red}{$\uparrow$2.4}}\\
        \midrule
        \multicolumn{14}{c}{\textit{Qwen2.5-3B as base model.}} \\
        \midrule
         SFT & 3.3 & 47.0 & 20.0 & 51.2 & 76.6 & 91.6 & 84.4 & 53.4 & 58.5 & 59.0 & 27.5 & 11.0 & 39.0\\
         UNDIAL & 3.3 & 50.8 & 17.5 & 49.2 & 77.4 & 91.2 & 83.2 & 53.2 \textcolor{green!50!black}{$\downarrow$0.2} & 56.1 & 59.5 & 28.9 & \underline{12.5} & 39.3 \textcolor{red}{$\uparrow$0.3}\\
         NEFTune & 3.3 & 51.2 & \underline{22.5} & 47.2 & 78.1 & 90.3 & 84.7 & 53.9 \textcolor{red}{$\uparrow$0.5} & 56.1 & 59.5 & 29.6 & 11.5 & 39.2 \textcolor{red}{$\uparrow$0.2}\\
         SPIN & 3.3 & \underline{51.6} & 20.0 & 49.6 & 77.4 & 91.2 & \textbf{86.1} & 54.1 \textcolor{red}{$\uparrow$0.7} & \underline{59.8} & \underline{60.6} & \underline{31.4} & 12.0 & \underline{41.0 \textcolor{red}{$\uparrow$2.0}}\\
         SSB & 3.3 & 51.4 & \underline{22.5} & \underline{52.0} & \underline{79.9} & \underline{92.0} & 84.9 & \underline{55.1 \textcolor{red}{$\uparrow$1.7}} & 54.3 & 52.9 & 30.4 & 9.3 & 36.7 \textcolor{green!50!black}{$\downarrow$2.3}\\
         \textbf{\modelname} & \textbf{6.7} & \textbf{52.0} & \textbf{22.5} & \textbf{53.9} & \textbf{80.9} & \textbf{92.4} & \underline{86.0} &  \textbf{56.3 \textcolor{red}{$\uparrow$2.9}} & \textbf{62.2} & \textbf{61.9} & \textbf{32.7} & \textbf{12.5} & \textbf{42.3 \textcolor{red}{$\uparrow$3.3}}\\
        \bottomrule
    \end{tabular}
    }

\end{table*}

%% file: content/4_experiments.tex
\section{Experiments}
\vspace{0.35\baselineskip}

This section presents the setup, main comparisons, module ablations, convergence dynamics, and WD-DS step-reduction analysis. Additional diagnostics are provided in Appendix~\ref{appendix: additionalexperiments}.

\subsection{Experimental setup}
\label{sec:experimental_setup}

We evaluate \modelname{} on mathematical reasoning, code generation, and logical reasoning against standard SFT, UNDIAL~\citep{dong2024undialselfdistillationadjustedlogits}, NEFTune~\citep{jain2023neftunenoisyembeddingsimprove}, SPIN~\citep{chen2024self}, and SSB~\citep{mitra2025semanticsoft}. \textbf{All compared models are trained for the same two-epoch budget.} \textbf{Because \modelname{} is an offline/off-policy post-training method, we focus the main comparison on offline or off-policy baselines for a fair comparison}. These methods use fixed supervision, alternate-context targets, or responses generated before the current update, rather than on-policy RL rollouts whose results depend strongly on online sampling and reward-model budgets. More details are provided in Appendix~\ref{app:dataconstruction}.

\subsection{Main results}
\label{sec:main_results}

\noindent\textbf{(1) \modelname{} gives a larger gain over the same two-epoch SFT baseline.}
We use the two-epoch SFT row as the reference and report improvements over it. On the three Qwen math blocks in Table~\ref{tab:main_results}, non-\modelname{} post-training baselines improve Math Avg by $+1.2$ points on average over SFT, whereas \modelname{} improves by $+4.6$ points on average, more than $3\times$ the mean competing gain. The same comparison is visible per backbone: Qwen3-4B-Base improves from $64.2\%$ to $\mathbf{69.1\%}$ ($+4.9$), Qwen3-8B-Base from $66.9\%$ to $\mathbf{72.8\%}$ ($+5.9$), and Qwen2.5-3B from $53.4\%$ to $\mathbf{56.3\%}$ ($+2.9$). Code and logic show the same direction: Qwen3-4B-Base gains $+4.4$ Code Avg and $+4.0$ Logic Avg, while Qwen3-8B-Base gains $+2.4$ Code Avg and $+7.4$ Logic Avg. Per-base logic results (LogiQA~2.0, ReClor) are tabulated in Table~\ref{tab:logic_results} of Appendix~\ref{app:logic_results}.

\noindent\textbf{(2) \modelname{} extends across the tested 3B--8B backbones and Qwen / Gemma families.}
The Math-Avg gain over the two-epoch SFT baseline is positive on every tested base: Qwen3-4B-Base $+4.9$, Qwen3-8B-Base $+5.9$, and Qwen2.5-3B $+2.9$ (Table~\ref{tab:main_results}); Gemma-3-4B-PT $+4.4$ (Table~\ref{tab:main_results_more_models} of Appendix~\ref{app:more_models}). In these single-run comparisons, \modelname{} obtains the strongest Math Avg among the listed baselines, suggesting that weak-driven logit mixing continues to surface useful corrective signal within this scale range. Appendix~\ref{app:checkpoint_pairing} also identifies a practical condition: the weak agent should retain sufficient entropy separation from the strong agent, since too-close pairings collapse toward averaging similar distributions.

\clearpage
\subsection{Module ablation: do WD-DS and WD-JT both matter?}
\label{sec:module_ablation}
\begin{wraptable}{r}{0.46\linewidth}
    \centering
    \input{table/ablation}
    \captionsetup{width=0.98\linewidth, justification=centering, singlelinecheck=false}
    \caption{\textbf{Ablation of individual modules}\\\textbf{of \modelname{} (Qwen3-4B-Base).}}
    \label{tab:ablation_compact}
\end{wraptable}
Table~\ref{tab:ablation_compact} isolates the contribution of the two modules on Qwen3-4B-Base. WD-DS alone improves the three-benchmark average from $54.4\%$ to $56.3\%$, showing that discrepancy-based selection already identifies more useful training examples. WD-JT alone reaches $58.2\%$ by converting weak-model discrepancies into corrective logit-mixing gradients. Combining WD-DS with WD-JT gives the full \modelname{} result of $59.9\%$ and improves AIME2025 from $13.3\%$ to $20.0\%$, indicating that WD-JT provides the main corrective training signal while WD-DS improves where that signal is spent.

\begin{figure*}[t]
\centering
\captionsetup{skip=2pt}
\begin{minipage}[t]{0.48\textwidth}
    \centering
    \includegraphics[width=\linewidth,height=3.25cm,keepaspectratio]{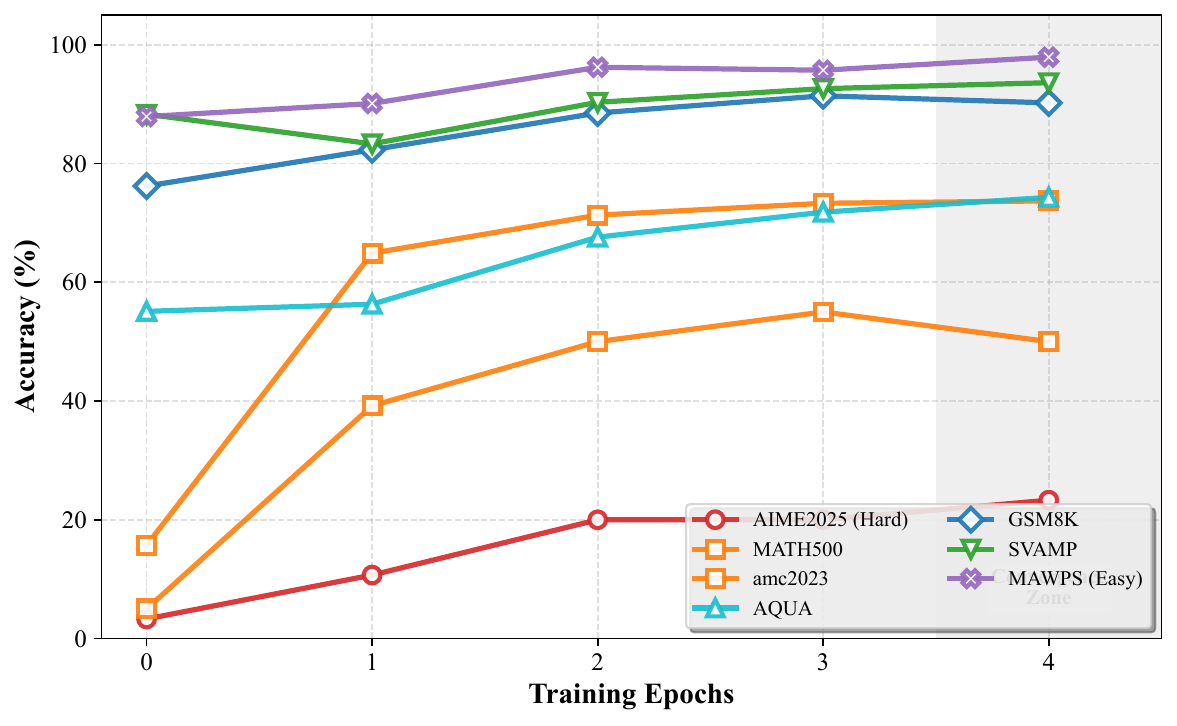}
    \caption{\textbf{Convergence of \modelname{} over four epochs.} The trajectory is reported as an extended diagnostic beyond the fixed two-epoch budget used for the main comparison.}
    \label{fig:convergence}
\end{minipage}
\hfill
\begin{minipage}[t]{0.48\textwidth}
    \centering
    \resizebox{0.90\linewidth}{!}{\input{content/ds_speed_tikz}}
    \caption{\textbf{Faster rise with WD-DS.} Avg. of MATH500 and GSM8K after the shared initialization; w/ and w/o denote with and without WD-DS.}
    \label{fig:ds_speed}
\end{minipage}
\vspace{-0.6\baselineskip}
\end{figure*}

\subsection{Convergence dynamics: when does \modelname{} converge?}
\label{sec:convergence_dynamics}
Figure~\ref{fig:convergence} visualizes an extended trajectory across seven math datasets beyond the two-epoch main comparison. The curves show rapid acquisition followed by stabilization: gains accumulate early, then marginal utility diminishes. The final stage exhibits \textbf{over-optimisation}: \textbf{AMC23} regresses after Epoch~3 and \textbf{GSM8K} shows late-stage volatility, while other datasets plateau. This pattern is compatible with the late-stage drift mechanism in Section~\ref{sec:theoretical_analysis}, but is not used as causal evidence or as a per-benchmark early-stopping rule.

\subsection{Does WD-DS reduce redundant weak-driven steps?}
\label{sec:cost_efficiency}

Weak-driven training adds a weak-model view of each update, so the key question is whether this signal is spent where it matters. WD-DS selects examples where weak--strong disagreement can produce corrective gradients and removes duplicate resampled indices so WD-JT runs on a shorter active schedule. Figure~\ref{fig:ds_speed} starts from the shared initialization checkpoint (MATH500 / GSM8K $=0.658/0.852$). The w/o WD-DS schedule uses 218 steps and reaches $0.694/0.873$; with WD-DS, the active schedule reaches $0.714/0.876$ with 123 steps, about $56\%$ of the full budget, and also surpasses w/o WD-DS at early checkpoints.

%% file: table/ablation.tex
\centering
\scriptsize
\setlength{\tabcolsep}{2.4pt}
\renewcommand{\arraystretch}{1.04}
\begin{tabular}{l c c c c}
    \toprule
    \textbf{Method} & \textbf{AIME} & \textbf{MATH} & \textbf{GSM8K} & \textbf{Avg.} \\
    \midrule
    Baseline & 13.3 & 66.1 & 83.9 & 54.4 \\
    + WD-DS & 13.3 & 69.4 & 86.2 & 56.3 \\
    + WD-JT & 16.7 & 70.2 & 87.6 & 58.2 \\
    \modelname & \textbf{20.0} & \textbf{71.3} & \textbf{88.5} & \textbf{59.9} \\
    \bottomrule
\end{tabular}

%% file: content/ds_speed_tikz.tex
\begin{tikzpicture}[x=1.10cm,y=1.10cm]
    \scriptsize
    \definecolor{fullblue}{RGB}{42,105,190}
    \definecolor{dedupred}{RGB}{190,45,55}
    \definecolor{sftgray}{RGB}{105,112,124}
    \definecolor{gridgray}{RGB}{220,224,230}
    \def\xInit{0.22}
    \def\xFifty{0.85}
    \def\xHundred{1.47}
    \def\xDedupFinal{1.76}
    \def\xOneFifty{2.08}
    \def\xTwoHundred{2.70}
    \def\xFullFinal{2.92}
    \def\ybase{0.34}
    \def\yscale{20.5}
    \foreach \y/\lab in {0.735/73.5,0.755/75.5,0.775/77.5,0.795/79.5} {
        \pgfmathsetmacro{\yy}{\ybase+\yscale*(\y-0.735)}
        \draw[gridgray, line width=0.22pt] (0.12,\yy) -- (3.05,\yy);
        \node[anchor=east] at (0.02,\yy) {\lab};
    }
    \draw[black, line width=0.35pt] (0.12,0.27) -- (3.05,0.27);
    \draw[black, line width=0.35pt] (0.12,0.27) -- (0.12,1.77);
    \foreach \x/\lab in {\xInit/0,\xFifty/50,\xDedupFinal/123,\xFullFinal/218} {
        \draw[black, line width=0.25pt] (\x,0.27) -- (\x,0.22);
        \node[anchor=north] at (\x,0.17) {\lab};
    }
    \node[anchor=north] at (1.58,-0.10) {steps};
    \node[rotate=90, anchor=south] at (-0.50,1.02) {Avg. acc.};
    \pgfmathsetmacro{\sA}{\ybase+\yscale*(0.755-0.735)}
    \pgfmathsetmacro{\fA}{\ybase+\yscale*(0.764-0.735)}
    \pgfmathsetmacro{\fB}{\ybase+\yscale*(0.773-0.735)}
    \pgfmathsetmacro{\fC}{\ybase+\yscale*(0.777-0.735)}
    \pgfmathsetmacro{\fD}{\ybase+\yscale*(0.776-0.735)}
    \pgfmathsetmacro{\fE}{\ybase+\yscale*(0.784-0.735)}
    \pgfmathsetmacro{\dA}{\ybase+\yscale*(0.773-0.735)}
    \pgfmathsetmacro{\dB}{\ybase+\yscale*(0.779-0.735)}
    \pgfmathsetmacro{\dC}{\ybase+\yscale*(0.795-0.735)}
    \pgfmathsetmacro{\sftA}{\ybase+\yscale*(0.745-0.735)}
    \pgfmathsetmacro{\sftB}{\ybase+\yscale*(0.737-0.735)}
    \pgfmathsetmacro{\sftC}{\ybase+\yscale*(0.750-0.735)}
    \pgfmathsetmacro{\sftD}{\ybase+\yscale*(0.751-0.735)}
    \pgfmathsetmacro{\sftE}{\ybase+\yscale*(0.762-0.735)}
    \draw[sftgray, line width=0.55pt, densely dashed]
        plot[smooth, tension=0.28] coordinates {(\xInit,\sA) (\xFifty,\sftA) (\xHundred,\sftB) (\xOneFifty,\sftC) (\xTwoHundred,\sftD) (\xFullFinal,\sftE)};
    \draw[fullblue, line width=0.62pt]
        plot[smooth, tension=0.25] coordinates {(\xInit,\sA) (\xFifty,\fA) (\xHundred,\fB) (\xOneFifty,\fC) (\xTwoHundred,\fD) (\xFullFinal,\fE)};
    \draw[dedupred, line width=0.62pt]
        plot[smooth, tension=0.25] coordinates {(\xInit,\sA) (\xFifty,\dA) (\xHundred,\dB) (\xDedupFinal,\dC)};
    \foreach \x/\y in {\xInit/\sA,\xFifty/\sftA,\xHundred/\sftB,\xOneFifty/\sftC,\xTwoHundred/\sftD,\xFullFinal/\sftE} \fill[sftgray] (\x,\y) circle (1pt);
    \foreach \x/\y in {\xInit/\sA,\xFifty/\fA,\xHundred/\fB,\xOneFifty/\fC,\xTwoHundred/\fD,\xFullFinal/\fE} \fill[fullblue] (\x,\y) circle (1.15pt);
    \foreach \x/\y in {\xInit/\sA,\xFifty/\dA,\xHundred/\dB,\xDedupFinal/\dC} \fill[dedupred] (\x,\y) circle (1.15pt);
    \node[dedupred, anchor=west, font=\tiny, fill=white, inner sep=0.4pt] at (1.20,1.66) {w/ WD-DS};
    \node[fullblue, anchor=west, font=\tiny, fill=white, inner sep=0.4pt] at (2.03,1.50) {w/o WD-DS};
    \node[sftgray, anchor=west, font=\tiny, fill=white, inner sep=0.4pt] at (2.55,0.98) {SFT};
\end{tikzpicture}

%% file: content/5_conclusion.tex
\vspace{-0.8\baselineskip}
\section{Discussion}
We introduced \emph{Weak-Driven Learning}, turning weak historical states into corrective training signal rather than obsolete checkpoints. Weak agents can still expose plausible but wrong alternatives on hard tokens, while the ground-truth cross-entropy target keeps optimization anchored to the correct answer. \modelname{} combines weak-driven logit mixing with discrepancy-aware selection, so extra weak-agent computation is spent where weak--strong differences are most informative.

This makes \modelname{} a practical post-training mechanism: it reuses artifacts already produced during training and requires no online rollouts, external strong teachers, or deployment-time inference changes. Our experiments show consistent gains across math, code, and logic, with WD-JT providing the main corrective signal and WD-DS improving where that signal is applied. Broader transfer and limitations are discussed in Appendix~\ref{app:limitations}.

%% file: content/6_appendex.tex
\section{Limitations and Broader Implications}
\label{app:limitations}

\paragraph{Limitations.}
While \modelname{} introduces no additional inference-time overhead, it uses an additional weak-agent forward pass during training. This makes the method most attractive in settings where modest extra training-time computation is acceptable in exchange for stronger post-training performance at deployment. The method is also designed for regimes where weak and strong agents retain useful token-level discrepancies; if the two agents become nearly identical, the additional corrective signal naturally diminishes. Our experiments focus on structured reasoning and coding tasks, where sharpening decisions against hard negatives is directly aligned with the evaluation objective. Extending the same principle to more open-ended generation settings may require task-adaptive mixing or softer objectives, which we leave to future work.

\paragraph{Broader implications.}
This work studies a post-training alternative in which useful supervision need not always originate from a stronger external source. In structured reasoning settings, our results suggest that weak agents produced during ordinary training can be reused as corrective signal when target-only SFT yields diminishing returns on hard samples. This provides a practical way to extract more value from existing training trajectories without introducing additional deployment-time inference mechanisms or requiring new external teachers. As with other methods that improve reasoning and code-generation ability, \modelname{} should be deployed with the same safeguards appropriate for capable language models.

\clearpage

\section{Additional Experiments}
\label{appendix: additionalexperiments}

\subsection{Freezing the Weak Agent during Weak-Driven Training}
\label{app:frozen_weak_agent}

Table~\ref{tab:frozen_weak_agent} isolates WD-JT without the WD-DS selection module and compares it with a variant that freezes the base weak-side agent. The frozen-base variant performs substantially worse, suggesting that WD-JT benefits from a weak--strong pair produced along the same training trajectory rather than a fixed base reference. When the weak side is jointly trained as a learner, it can also lose mathematical competence, weakening the corrective contrast that WD-JT relies on.

\begin{table}[h]
\centering
\small
\setlength{\tabcolsep}{5pt}
\caption{\textbf{Effect of freezing the weak-side agent.} Results are MATH500/GSM8K pass@1 accuracy (\%) at relative checkpoints after the shared initialization.}
\label{tab:frozen_weak_agent}
\begin{tabular}{c c c}
\toprule
\textbf{ck-rel} & \textbf{WD-JT only (w/o WD-DS)} & \textbf{Frozen base} \\
\midrule
50  & 66.6 / 86.2 & 62.0 / 79.6 \\
100 & 69.0 / 86.4 & 62.2 / 82.9 \\
150 & 68.0 / 86.6 & 63.2 / 85.4 \\
200 & \textbf{69.4 / 87.6} & 63.8 / 84.9 \\
final & 68.8 / 86.4 & 63.2 / 86.1 \\
\bottomrule
\end{tabular}
\end{table}

\subsection{Results on Gemma-3-4B-PT}
\label{app:more_models}

To test the weak-driven mixed-logit mechanism beyond Qwen-series models, we extend the main-paper benchmark (Table~\ref{tab:main_results}) to Gemma-3-4B-PT and report results in Table~\ref{tab:main_results_more_models}. \modelname{} attains the best Math-Avg ($+4.4$ over SFT) and Code-Avg ($+3.9$ over SFT).
\input{table/main_results_more_models}

\subsection{Logic-Reasoning Results}
\label{app:logic_results}

To keep the main-paper Table~\ref{tab:main_results} focused on math and code, we report the logic-reasoning benchmarks (LogiQA~2.0 and ReClor) separately in Table~\ref{tab:logic_results}. Across all four base models, \modelname{} attains the best LogiQA~2.0, ReClor, and logic-average accuracy, suggesting that the weak-driven mixed-logit mechanism extends beyond numerical and coding domains within structured reasoning tasks.
\input{table/logic_results}

\subsection{Logit statistics: does WD-JT suppress hard negatives?}
\label{sec:logit_statistics}
\input{table/logit_analysis}
Table~\ref{tab:logit_analysis} provides a logit-space diagnostic consistent with the mechanism in Section~\ref{sec:theoretical_analysis}. After standard SFT, target logits are already large, so further improvement is unlikely to come mainly from pushing the correct token higher. Instead, \modelname{} reduces the non-target logit mean from $2.09$ to $0.90$ ($-56.9\%$), while the target logit changes only slightly ($35.88\to36.10$). This suppression-dominant behavior expands the target-to-background gap by $+1.41$ logit points; we treat it as mechanistic evidence consistent with WD-JT increasing gradient signal on weak-revealed hard negatives, rather than as standalone causal proof.

To provide a more complete view of the same dynamics, Table~\ref{tab:logit_analysis_extended} reports additional statistics from the logit analysis report. Beyond the non-target logit mean reduction above, the extended metrics compare the frozen weak reference with the strong branch before and after WD-JT; only the strong branch is updated.
\paragraph{Implementation details.}
All logit statistics are computed from 200 randomly sampled examples drawn from the AM-1.4M training dataset used in our experiments. We run forward passes for the frozen weak reference and for the strong pre-/post-WD-JT checkpoints on these examples and aggregate statistics across the sampled set. The centered-norm definition used in the table follows Appendix~\ref{sec:appendix_joint_theory} (Eq.~\eqref{eq:appendix_centered_norm}).
\input{table/logit_analysis_extended}

\subsection{Weak--strong entropy separation}
\label{app:checkpoint_pairing}

We fix the strong branch to \texttt{ck-218} and vary the weak-side checkpoint $m_1$ to diagnose entropy separation. When $\Delta H$ is small, the two branches have similar uncertainty patterns and logit mixing adds little corrective signal; \texttt{ck-200} drops to $0.484/0.742$. Once the gap is sufficient, the weak branch can supply uncertainty without destabilizing training because CE still anchors the target. The close results of base and \texttt{ck-50} suggest a threshold effect: sufficient separation matters, but larger gaps do not continuously improve fusion.

\begin{figure}[H]
    \centering
    \includegraphics[width=0.78\linewidth]{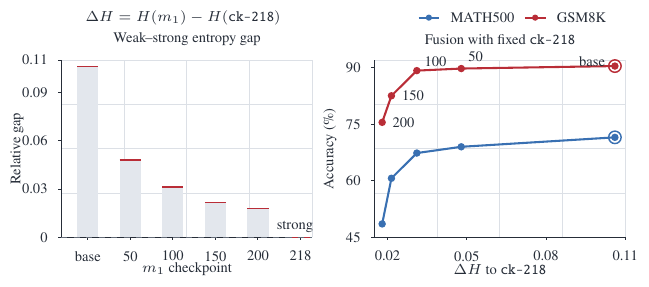}
    \caption{\textbf{Diagnostic of weak--strong entropy separation.} Small-gap pairings degrade, while base and \texttt{ck-50} remain close, suggesting a threshold rather than monotonic gains from larger $\Delta H$.}
    \label{fig:checkpoint_pairing}
\end{figure}


\subsection{Fixed hyperparameters and post-hoc sensitivity}
\label{app:hparam_protocol}

To avoid benchmark-level model selection, the headline results in Table~\ref{tab:main_results} use a single fixed \modelname{} setting across all reported models and domains: $\alpha=0.1$, $\beta=0.8$, $\gamma=0.1$, and $\lambda=0.5$. These values are treated as method defaults rather than per-benchmark tuned hyperparameters. In particular, we do not replace the main results with the best configuration found in the sweeps below; the $\lambda$ sweep later peaks at $\lambda=0.42$, while the headline runs remain at the fixed default $\lambda=0.5$. Tables~\ref{tab:hyperparam_sensitivity} and~\ref{tab:sensitive_of_lambda} are therefore intended as post-hoc sensitivity analyses that characterize robustness around the default operating point, not as a model-selection procedure for the systems reported in Table~\ref{tab:main_results}. Three observations support this interpretation:

\textbf{(i) The default operating point is not at a sharp peak.} The $\lambda$ sweep (Table~\ref{tab:sensitive_of_lambda}) shows a broad inverted-U with a plateau over $\lambda \in [0.42, 0.50]$, and the $(\alpha,\beta,\gamma)$ sensitivity (Table~\ref{tab:hyperparam_sensitivity}) places the default configuration within $\sim$2\,pp Math-Avg of the neighboring configurations tested. Points inside the plateau yield results close to the fixed-default headline.

\textbf{(ii) Non-default hyperparameter choices remain competitive.} The worst-performing configuration in our $\lambda$ sweep ($\lambda{=}0.9$, Math Avg.\ $67.6\%$ on Qwen3-4B-Base) still exceeds the strongest non-\modelname{} baseline in Table~\ref{tab:main_results} (SSB, $66.1\%$) by $\sim$1.5\,pp and exceeds plain SFT ($64.2\%$) by $>$3\,pp. This suggests that the observed gain is not solely a product of a narrow hyperparameter peak.

\textbf{(iii) The high-performing $\lambda$ region is consistent with the theory.} The centered-norm crossover formula in Appendix~\ref{sec:appendix_joint_theory} (Eq.~\eqref{eq:appendix_lambda_cross_est}) predicts a high-performing mixing region near $\lambda \approx 0.49$ from the empirical centered-norm ratio of the two agents, in close agreement with the observed plateau around $\lambda \in [0.42,0.50]$. This is a mechanism-level consistency check, not a selection rule for the headline models.

\subsection{Sensitivity of Difficulty, Consolidation, and Repair Coefficients.}
\input{table/hyperparam_sensitivity}
We further investigate the impact of the consolidation coefficient $\alpha$, the difficulty coefficient $\beta$, and the regression-repair coefficient $\gamma$. Table~\ref{tab:hyperparam_sensitivity} reports a post-hoc sensitivity check on Qwen3-4B-Base and Qwen2.5-3B. On Qwen3-4B-Base we observe a distinct trade-off between standard mathematical proficiency (MATH 500) and complex reasoning capability (AIME 2025).

Configuration C ($\alpha=0.1, \beta=0.9, \gamma=0$), which disables the regression-repair signal, achieves the highest accuracy on MATH 500 (\textbf{70.2\%}). However, its performance on the more challenging AIME benchmark drops significantly to $10.3\%$. This suggests that while stronger emphasis on base difficulty ($\beta=0.9$) helps with standard problems, it can still lead to optimization stagnation on harder tasks without regression repair.

In contrast, the fixed default configuration ($\alpha=0.1, \beta=0.8, \gamma=0.1$) introduces a controlled regression-repair emphasis. Although this slightly reduces MATH performance ($70.2\% \rightarrow 68.2\%$), it yields a notable \textbf{+6.4\%} gain on AIME ($10.3\% \rightarrow 16.7\%$). This suggests that the additional signal induced by $\gamma$ is useful for harder reasoning cases and that the default setting is not an isolated peak chosen solely from the aggregate benchmark score.

\paragraph{Qwen2.5-3B.} On the smaller Qwen2.5-3B base, this post-hoc sweep does not register any correct AIME~2025 answer across the six tested $\alpha/\beta/\gamma$ variants (shown as ``-'' in Table~\ref{tab:hyperparam_sensitivity}), so the variants cannot be discriminated on this benchmark. The MATH 500 column nevertheless shows that the fixed default ($\alpha=0.1,\beta=0.8,\gamma=0.1$) remains competitive and attains the best score (\textbf{48.2\%}), outperforming the regression-repair-disabled Config~C ($47.6\%$) and the difficulty-only Config~D ($47.0\%$). The relative gap between the default and Config~C is smaller on 3B than on 4B (+0.6 vs.\ $-2.0$ on MATH 500), consistent with the interpretation that the regression-repair term is most useful when the base model can recover hard-task signal in the first place.

\subsection{Sensitivity Analysis on Mixing Coefficient \texorpdfstring{$\lambda$}{lambda}}
The mixing coefficient $\lambda$ explicitly controls the relative weight between the weak and strong logits in mixed-logit training, with fused logits $z_{\text{mix}} = (1-\lambda) z_1 + \lambda z_2$. Smaller $\lambda$ assigns more influence to the weaker model, while larger $\lambda$ emphasizes the stronger model's direct fit to the targets.

We sweep $\lambda$ over $[0.1, 0.9]$ with a finer grid in $[0.4, 0.6]$ after the headline experiments have been fixed. Table~\ref{tab:sensitive_of_lambda} shows a broad inverted U-shape: the best average performance occurs at $\lambda=0.42$ (Avg.\ $75.5\%$), and a strong plateau persists in $\lambda \in [0.42, 0.50]$, which includes the fixed default $\lambda=0.5$. All benchmarks are evaluated with three-run averages; minor fluctuations on small sets (e.g., AIME25, 30 problems) are still expected, but the high-performing region remains stable.

At the extremes, the behavior aligns with the mechanism. As $\lambda \to 1$, $z_{\text{mix}} \approx z_2$ and the mixed-logit objective reduces to relying on the strong model, so the compensatory interaction weakens and accuracy drops (e.g., Avg.\ $67.6\%$ at $\lambda=0.9$). As $\lambda \to 0$, the weak model dominates the fused logits, reducing effective target learning and risking underfitting; this is especially pronounced for weaker base models (e.g., at $\lambda=0.3$, a weaker model attains only $13\%$ accuracy on MATH500).

\paragraph{Why does the optimum sit near $\lambda \approx 0.42$ rather than $\lambda \to 1$?} A natural intuition is that more weight on the strong model ($\lambda \to 1$) should be better, since the strong model alone is already correct on most tokens. But \modelname{}'s benefit comes from the \emph{weak} agent's contribution to the fused distribution: the weak agent's softer decision boundary preserves probability mass on hard negatives, which is exactly the signal needed on hard samples (Theorem~1). At $\lambda \to 1$ this weak--strong interaction vanishes and the loss reverts to standard target-only SFT. The optimum at $\lambda \approx 0.42$ therefore reflects a balance between two distinct roles: the strong agent contributes the \emph{target direction} (we still want the loss to push toward $y$), while the weak agent contributes \emph{uncertainty over hard negatives}. \modelname{} is not equivalent to a regulariser that shrinks logits toward zero: the weak agent supplies a structurally informed direction, not isotropic noise, which is consistent with our gradient-share crossover prediction near $\lambda \approx 0.49$ (Eq.~\eqref{eq:appendix_lambda_cross_est}).

\paragraph{Numerical consistency with the gradient-share crossover.}
Following the theoretical analysis in Appendix~\ref{sec:appendix_joint_theory}, we can estimate the effective sensitivity ratio using the centered-norm proxy:
\begin{equation}
    \alpha \approx \left(\frac{\|\tilde{z}_2\|_2}{\|\tilde{z}_1\|_2}\right)^2.
    \label{eq:appendix_alpha_est}
\end{equation}
From the logits evaluation report, the strong pre-update centered norm is $\|\tilde{z}_2\|_2 \approx 1240.10$ and the weak-side initialization centered norm is $\|\tilde{z}_1\|_2 \approx 1191.33$, giving $\alpha \approx 1.08$.
Plugging into the crossover formula in Eq.~\eqref{eq:appendix_lambda_cross},
\begin{equation}
    \lambda_{\text{cross}} \approx \frac{1}{1+\sqrt{\alpha}} \approx 0.490,
    \label{eq:appendix_lambda_cross_est}
\end{equation}
which lies close to the empirically strongest region ($\lambda \in [0.42, 0.50]$). We stress that this is a \emph{heuristic consistency check}: $\alpha$ is phase-dependent and the linearization is local, so the theory predicts a broad optimum region rather than a sharp inversion point.
\input{table/all_lambda}

\subsection{Pass@\textit{k} Comparison with DPO}
\label{app:pass_at_k}

To compare \modelname{} against a preference-based baseline under multi-sample evaluation, we evaluate both \modelname{}\,(Qwen3-4B-Base) and a DPO baseline trained from the same SFT warm-start on the same math data. The DPO preference pairs are constructed by sampling $4$ rollouts per prompt at $T{=}0.8$ from the SFT-warm-start policy and using \texttt{math\_verify} against the ground-truth boxed answer to label each rollout as correct or incorrect. We use the first verified correct rollout in the fixed sampling order as the chosen response and the first verified incorrect rollout as the rejected response; prompts with no valid pair are skipped. No external reward model or process-reward signal is used. For evaluation, we draw $n=64$ rollouts per problem at $T=0.7$, top-$p=0.95$, and report the unbiased pass@$k$ estimator of \citet{chen2021evaluatinglargelanguagemodels}, $\widehat{\mathrm{pass}@k}=\mathbb{E}_{\text{problem}}[1-\binom{n-c}{k}/\binom{n}{k}]$ (where $c$ is the number of correct rollouts), for $k\in\{1,4,8,16,32,64\}$ on AIME~2025, AMC23, and MATH500.

\begin{table}[h]
\centering
\small
\setlength{\tabcolsep}{4pt}
\caption{\textbf{Pass@$k$ evaluation of \modelname{} vs.\ DPO on Qwen3-4B-Base.} Sampling at $T=0.7$, top-$p=0.95$, $n=64$ samples per problem; the $k{=}1$ entry is equivalent to mean@64. \textbf{Bold} marks the higher value in each (benchmark, $k$) cell.}
\label{tab:pass_at_k}
\begin{tabular}{c c c c c c c}
\toprule
\multirow{2}{*}{$k$} & \multicolumn{2}{c}{\textbf{AIME~2025}} & \multicolumn{2}{c}{\textbf{AMC23}} & \multicolumn{2}{c}{\textbf{MATH500}} \\
\cmidrule(lr){2-3} \cmidrule(lr){4-5} \cmidrule(lr){6-7}
 & \modelname{} & DPO & \modelname{} & DPO & \modelname{} & DPO \\
\midrule
1  & \textbf{17.97} & 14.20 & \textbf{49.49} & 42.80 & \textbf{75.32} & 70.90 \\
4  & \textbf{24.69} & 22.40 & \textbf{63.25} & 56.70 & \textbf{84.36} & 84.00 \\
8  & \textbf{26.59} & 24.90 & \textbf{68.06} & 61.00 & 87.20 & \textbf{87.40} \\
16 & \textbf{28.57} & 27.30 & \textbf{71.79} & 65.20 & 89.48 & \textbf{89.90} \\
32 & \textbf{30.83} & 29.10 & \textbf{75.13} & 69.80 & 91.49 & \textbf{91.80} \\
64 & \textbf{33.33} & 30.00 & \textbf{77.50} & 75.00 & 93.20 & \textbf{93.40} \\
\bottomrule
\end{tabular}
\end{table}

\paragraph{\modelname{} outperforms DPO on hard reasoning benchmarks across $k$.}
On AIME~2025 and AMC23 \modelname{} beats DPO at every $k$: $+3.8$\,pp at pass@1 ($17.97$ vs.\ $14.20$) and $+6.7$\,pp at pass@1 ($49.49$ vs.\ $42.80$) respectively, with the lead remaining positive at the diversity ceiling $k{=}64$ ($+3.3$ on AIME, $+2.5$ on AMC23). On the easier MATH500, \modelname{} leads at pass@1 / pass@4 ($+4.4$ / $+0.4$) but DPO catches up for $k\ge 8$ within $0.4$\,pp as both methods approach the $90$\,\% ceiling. This pattern is consistent with \modelname{} helping most where hard-reasoning gradients remain informative, while preference-based diversity can close the gap on easier benchmarks at larger $k$.

\paragraph{The advantage is sharpest at low $k$.}
The relative gap between \modelname{} and DPO is largest at small $k$ and shrinks as $k$ grows: on AIME~2025 the ratio (\modelname{}/DPO) is $1.27\times$ at $k{=}1$ and $1.11\times$ at $k{=}64$; on AMC23 it shrinks from $1.16\times$ to $1.03\times$; on MATH500 from $1.06\times$ to $1.00\times$. This suggests that \modelname{} improves the single-decode regime, which is relevant when best-of-$N$ rejection sampling is expensive. This is consistent with the gradient-amplification analysis in Section~\ref{sec:theoretical_analysis}: by increasing gradient signal on weak-revealed hard negatives, \modelname{} may sharpen the per-token decision boundary and improve top-1 accuracy.

\paragraph{Diversity gain (pass@$64$ -- pass@$1$).}
The pass@$k$ growth $\Delta = \mathrm{pass@64} - \mathrm{pass@1}$ measures how much additional problem coverage the model finds via diverse sampling. \modelname{} grows by $\Delta_{\text{AIME}}{=}{+15.4}$, $\Delta_{\text{AMC23}}{=}{+28.0}$, $\Delta_{\text{MATH500}}{=}{+17.9}$\,pp; DPO grows by $\Delta_{\text{AIME}}{=}{+15.8}$, $\Delta_{\text{AMC23}}{=}{+32.2}$, $\Delta_{\text{MATH500}}{=}{+22.5}$. DPO has higher diversity gains across the board, while \modelname{} starts from a stronger pass@1 baseline on AIME and AMC23; as a result, DPO's additional sampling does not close the absolute gap on those two hard benchmarks by pass@$64$.

\subsection{Training Cost and WD-DS Runtime}
\label{app:wd_ds_runtime}

Table~\ref{tab:baseline_runtime} reports measured wall-clock time for timed post-training runs on Qwen3-4B-Base after the same SFT warm-start. These numbers are implementation-level measurements, not a FLOP-normalized comparison: rollout-heavy methods include their own generation stages, while \modelname{} includes the WD-DS active schedule and dual forward pass during WD-JT. The \modelname{} entry uses the deduplicated WD-DS schedule reported in Section~\ref{sec:cost_efficiency}, replacing the earlier full-schedule runtime with the measured ck-123 runtime.

\begin{table}[h]
\centering
\small
\setlength{\tabcolsep}{5pt}
\caption{\textbf{Measured training time on Qwen3-4B-Base.} Wall-clock reports total measured training time after the shared SFT warm-start; it is not FLOP-normalized across methods with different rollout or dual-forward costs. Benefit is the Math-Avg pass@1 gain over SFT in Table~\ref{tab:main_results}.}
\label{tab:baseline_runtime}
\begin{tabular}{l c c}
\toprule
Method & Total training time & Math-Avg $\Delta$ \\
\midrule
SFT       & 4h30m        & -- \\
NEFTune   & 4h30m        & $+0.7$ \\
UNDIAL    & 4h31m        & $-2.9$ \\
SSB       & $\sim$18h    & $+1.9$ \\
SPIN      & $\sim$7h53m  & $+2.4$ \\
\textbf{\modelname{} (dedup WD-DS)} & \textbf{4h49min} & $\mathbf{+4.9}$ \\
\bottomrule
\end{tabular}
\end{table}

\newpage
\section{Supplementary Theory: Gradient Amplification under Weak-Driven Logit Mixing}
\label{sec:appendix_grad_amp}

\paragraph{Notation.}
Fix a context $x$ and a target token index $y$, and analyze a single mixed-logit training step; we omit the conditioning on $x$ for readability.
The weak and strong agents $\mathcal{M}_{\mathrm{weak}}$ and $\mathcal{M}_{\mathrm{strong}}$ produce logits $z_{\mathrm{weak}}(x),z_{\mathrm{strong}}(x)\in\mathbb{R}^{|\mathcal{V}|}$, and the mixed logits are
\begin{equation}
    z_{\text{mix}}=(1-\lambda)z_{\mathrm{weak}}+\lambda z_{\mathrm{strong}},\quad \lambda\in[0,1].
    \label{eq:appendix_grad_mix}
\end{equation}
For any logit map $z(x)$, define $P_z(\cdot\mid x)=\mathrm{Softmax}(z(x))$ and use the shorthand $P_z(k)$; let $P_{\mathrm{mix}}=P_{z_{\text{mix}}}$ and $e_y$ denote the one-hot target vector.
For any negative token $k\neq y$, the target margin is $m_k(z)=z[y]-z[k]$.

\subsection{Setup and Baseline}
We compare the mixed-logit loss with standard SFT on the strong agent:
\begin{equation}
    \mathcal{L}_{\mathrm{mix}}=-\log P_{\mathrm{mix}}(y),\qquad
    \mathcal{L}_{\mathrm{SFT}}=-\log P_{\mathrm{strong}}(y),
    \label{eq:appendix_grad_ce}
\end{equation}
where $P_{\mathrm{strong}}=P_{z_{\mathrm{strong}}}$.
The gradient with respect to the fused logits is
\begin{equation}
    g_{\text{mix}}=\nabla_{z_{\text{mix}}}\mathcal{L}_{\mathrm{mix}}=P_{\mathrm{mix}}(\cdot)-e_y.
    \label{eq:appendix_grad_residual}
\end{equation}
In standard SFT on the strong agent alone, the corresponding gradient is
\begin{equation}
    g_{\text{sft}}=\nabla_{z_{\mathrm{strong}}}\mathcal{L}_{\mathrm{SFT}}=P_{\mathrm{strong}}(\cdot)-e_y.
    \label{eq:appendix_grad_sft}
\end{equation}
By the chain rule, the branch-level logit gradients induced by the mixed objective are
\begin{equation}
    \nabla_{z_{\mathrm{weak}}}\mathcal{L}_{\mathrm{mix}}=(1-\lambda)g_{\text{mix}},\quad
    \nabla_{z_{\mathrm{strong}}}\mathcal{L}_{\mathrm{mix}}=\lambda g_{\text{mix}}.
    \label{eq:appendix_grad_model}
\end{equation}

\subsection{Log-Odds and Margin Contraction}
\begin{definition}[Target margin]
For any negative token $k\neq y$, define
\begin{equation}
    m_k(z)=z[y]-z[k].
    \label{eq:appendix_margin_def}
\end{equation}
\end{definition}

\begin{lemma}[Softmax log-odds]
For any $k\neq y$, the log-odds under logits $z$ satisfy
\begin{equation}
    \log\frac{P_z(k)}{P_z(y)}=-m_k(z).
    \label{eq:appendix_logodds}
\end{equation}
\end{lemma}
\begin{proof}
By definition, $P_z(k)/P_z(y)=\exp(z[k]-z[y])=\exp(-m_k(z))$.
\end{proof}

\begin{lemma}[Margin mixing]
For any $k\neq y$, the mixed margin equals the convex combination
\begin{equation}
    m_k(z_{\text{mix}})=(1-\lambda)m_k(z_{\mathrm{weak}})+\lambda m_k(z_{\mathrm{strong}}).
    \label{eq:appendix_margin_mix}
\end{equation}
\end{lemma}
\begin{proof}
Using \eqref{eq:appendix_grad_mix} and \eqref{eq:appendix_margin_def}, we expand
\begin{align}
    m_k(z_{\text{mix}})
    &=z_{\text{mix}}[y]-z_{\text{mix}}[k] \nonumber\\
    &=\big[(1-\lambda)z_{\mathrm{weak}}[y]+\lambda z_{\mathrm{strong}}[y]\big]
      -\big[(1-\lambda)z_{\mathrm{weak}}[k]+\lambda z_{\mathrm{strong}}[k]\big] \nonumber\\
    &=(1-\lambda)(z_{\mathrm{weak}}[y]-z_{\mathrm{weak}}[k])
      +\lambda(z_{\mathrm{strong}}[y]-z_{\mathrm{strong}}[k]) \nonumber\\
    &=(1-\lambda)m_k(z_{\mathrm{weak}})+\lambda m_k(z_{\mathrm{strong}}).
    \label{eq:appendix_margin_mix_proof}
\end{align}
\end{proof}

\subsection{Negative-Gradient Amplification}
\begin{definition}[Hard-negative set]
Define the hard-negative set
\begin{equation}
    \mathcal{H}=\{k\neq y:m_k(z_{\mathrm{weak}})<m_k(z_{\mathrm{strong}})\},
    \label{eq:appendix_hard_set}
\end{equation}
i.e., tokens for which the weak agent has a smaller target-vs-negative margin.
\end{definition}

\begin{lemma}[Relative probability increase on hard negatives]
For any $k\in\mathcal{H}$,
\begin{equation}
    \frac{P_{\mathrm{mix}}(k)}{P_{\mathrm{mix}}(y)}
    >
    \frac{P_{\mathrm{strong}}(k)}{P_{\mathrm{strong}}(y)}.
    \label{eq:appendix_rel_prob_increase}
\end{equation}
\end{lemma}
\begin{proof}
By \eqref{eq:appendix_margin_mix}, $m_k(z_{\text{mix}})<m_k(z_{\mathrm{strong}})$ for $k\in\mathcal{H}$.
Applying \eqref{eq:appendix_logodds} yields a larger log-odds ratio for the mixed logits.
\end{proof}

\begin{corollary}[Sufficient condition for per-token amplification]
For any $k\in\mathcal{H}$, the negative-token gradient on fused logits satisfies
\begin{equation}
    P_{\mathrm{mix}}(k)\ge P_{\mathrm{strong}}(k)
    \quad\text{whenever}\quad
    \frac{P_{\mathrm{mix}}(y)}{P_{\mathrm{strong}}(y)}
    \ge \exp\!\left(-(1-\lambda)\Delta m_k\right),
    \label{eq:appendix_amp_condition}
\end{equation}
where $\Delta m_k=m_k(z_{\mathrm{strong}})-m_k(z_{\mathrm{weak}})>0$.
\end{corollary}
\begin{proof}
Using log-odds,
\begin{align}
    \frac{P_{\mathrm{mix}}(k)}{P_{\mathrm{strong}}(k)}
    &=\frac{P_{\mathrm{mix}}(y)\exp(-m_k(z_{\text{mix}}))}
            {P_{\mathrm{strong}}(y)\exp(-m_k(z_{\mathrm{strong}}))} \nonumber\\
    &=\frac{P_{\mathrm{mix}}(y)}{P_{\mathrm{strong}}(y)}
      \exp\!\left(m_k(z_{\mathrm{strong}})-m_k(z_{\text{mix}})\right).
    \label{eq:appendix_ratio_step1}
\end{align}
By \eqref{eq:appendix_margin_mix},
\begin{align}
    m_k(z_{\mathrm{strong}})-m_k(z_{\text{mix}})
    &=(1-\lambda)\big(m_k(z_{\mathrm{strong}})-m_k(z_{\mathrm{weak}})\big) \nonumber\\
    &=(1-\lambda)\Delta m_k.
    \label{eq:appendix_ratio_step2}
\end{align}
Thus
\begin{equation}
    \frac{P_{\mathrm{mix}}(k)}{P_{\mathrm{strong}}(k)}
    =
    \frac{P_{\mathrm{mix}}(y)}{P_{\mathrm{strong}}(y)}
    \exp\!\left((1-\lambda)\Delta m_k\right).
    \label{eq:appendix_ratio_expand}
\end{equation}
Solving for $P_{\mathrm{mix}}(k)\ge P_{\mathrm{strong}}(k)$ gives \eqref{eq:appendix_amp_condition}.
\end{proof}

\begin{theorem}[Total negative-mass increase under uniform margin shrinkage]
If $m_k(z_{\mathrm{weak}})\le m_k(z_{\mathrm{strong}})$ for all $k\neq y$, then
\begin{equation}
    P_{\mathrm{mix}}(y)\le P_{\mathrm{strong}}(y),
    \qquad
    \sum_{k\neq y}P_{\mathrm{mix}}(k)\ge\sum_{k\neq y}P_{\mathrm{strong}}(k).
    \label{eq:appendix_total_neg_mass}
\end{equation}
\end{theorem}
\begin{proof}
By \eqref{eq:appendix_margin_mix} and the assumption,
\begin{equation}
    m_k(z_{\text{mix}})=(1-\lambda)m_k(z_{\mathrm{weak}})+\lambda m_k(z_{\mathrm{strong}})
    \le m_k(z_{\mathrm{strong}})
    \quad \forall k\neq y.
    \label{eq:appendix_margin_shrink}
\end{equation}
Applying \eqref{eq:appendix_logodds} yields
\begin{equation}
    \exp\!\big(-m_k(z_{\text{mix}})\big)
    \ge
    \exp\!\big(-m_k(z_{\mathrm{strong}})\big)
    \quad \forall k\neq y.
    \label{eq:appendix_exp_compare}
\end{equation}
Summing over $k\neq y$ gives
\begin{equation}
    \sum_{k\neq y}\exp\!\big(-m_k(z_{\text{mix}})\big)
    \ge
    \sum_{k\neq y}\exp\!\big(-m_k(z_{\mathrm{strong}})\big).
    \label{eq:appendix_sum_compare}
\end{equation}
Since
\begin{equation}
    P_z(y)=\left(1+\sum_{k\neq y}\exp\!\big(-m_k(z)\big)\right)^{-1},
    \label{eq:appendix_py_denominator}
\end{equation}
the denominator for $P_{\mathrm{mix}}(y)$ is no smaller than that for $P_{\mathrm{strong}}(y)$, and therefore $P_{\mathrm{mix}}(y)\le P_{\mathrm{strong}}(y)$. The second inequality follows from probabilities summing to one:
\begin{equation}
    \sum_{k\neq y}P_{\mathrm{mix}}(k)=1-P_{\mathrm{mix}}(y)
    \ge 1-P_{\mathrm{strong}}(y)
    =\sum_{k\neq y}P_{\mathrm{strong}}(k).
    \label{eq:appendix_neg_mass_compare}
\end{equation}
\end{proof}

\begin{proposition}[Logit updates emphasize negative suppression]
For any negative token $k\neq y$, the local logit update for branch $i\in\{\mathrm{weak},\mathrm{strong}\}$ under the mixed objective is
\begin{equation}
    \Delta z_i[k]\approx-\eta s_i P_{\mathrm{mix}}(k),\quad
    s_{\mathrm{weak}}=1-\lambda,\;s_{\mathrm{strong}}=\lambda,
    \label{eq:appendix_neg_update}
\end{equation}
while for the target token,
\begin{equation}
    \Delta z_i[y]\approx\eta s_i(1-P_{\mathrm{mix}}(y)).
    \label{eq:appendix_pos_update}
\end{equation}
Thus any increase in $P_{\mathrm{mix}}(k)$ directly amplifies suppression of hard negatives, while a decrease in $P_{\mathrm{mix}}(y)$ strengthens the upward push on the target logit.
\end{proposition}
\begin{proof}
From \eqref{eq:appendix_grad_model}, the branch-$i$ logit gradient satisfies
\begin{equation}
    \nabla_{z_i}\mathcal{L}_{\mathrm{mix}}=s_i g_{\text{mix}},
    \label{eq:appendix_prop_grad}
\end{equation}
where $g_{\text{mix}}=P_{\mathrm{mix}}(\cdot)-e_y$.
For a negative token $k\neq y$,
\begin{equation}
    \frac{\partial \mathcal{L}_{\mathrm{mix}}}{\partial z_i[k]}=s_iP_{\mathrm{mix}}(k),
    \label{eq:appendix_prop_neg_grad}
\end{equation}
and for the target token,
\begin{equation}
    \frac{\partial \mathcal{L}_{\mathrm{mix}}}{\partial z_i[y]}=s_i(P_{\mathrm{mix}}(y)-1).
    \label{eq:appendix_prop_pos_grad}
\end{equation}
Under a first-order update $\Delta z_i\approx-\eta\nabla_{z_i}\mathcal{L}_{\mathrm{mix}}$, we obtain \eqref{eq:appendix_neg_update} and \eqref{eq:appendix_pos_update}.
\end{proof}

\begin{remark}[Consistency with logit statistics]
The mechanism above provides a local explanation for the logit statistics reported in Table~\ref{tab:logit_analysis} and Table~\ref{tab:logit_analysis_extended}: increased negative probability mass amplifies downward updates on incorrect tokens, while a lower mixed target probability strengthens the upward update on the correct token. Mean shifts can further accumulate along shift-invariant directions, as analyzed in Appendix~\ref{sec:appendix_joint_theory}.
\end{remark}

\newpage
\section{Additional Dynamics Diagnostics for Mixed-Logit Training}
\label{sec:appendix_joint_theory}

\subsection{Problem Setting and Notation}
\paragraph{Notation.}
Fix a context $x$ and a target index $y$, and analyze a single WD-JT mixed-logit update; we omit the conditioning on $x$ for readability. The weak and strong agents form a training-time pair and produce logits $z_{\mathrm{weak}}(x), z_{\mathrm{strong}}(x) \in \mathbb{R}^{|\mathcal{V}|}$. We write branch-level derivatives with respect to both logit maps because the mixed objective backpropagates through the paired logits, while the strong branch is the model retained for evaluation and deployment. The mixed logits are
\begin{equation}
    z_{\text{mix}} = (1-\lambda) z_{\mathrm{weak}} + \lambda z_{\mathrm{strong}}, \quad \lambda \in [0,1].
    \label{eq:appendix_mix}
\end{equation}
For any logit map $z(x)$, define $P_z(\cdot\mid x)=\mathrm{Softmax}(z(x))$ and use the shorthand $P_z(k)$; let $P_{\mathrm{mix}} = P_{z_{\text{mix}}}$ and $e_y$ denote the one-hot target. We use a subscript $i \in \{\mathrm{weak}, \mathrm{strong}\}$ to index the two models in generic expressions. Let $\mathbf{1}$ be the all-ones vector in $\mathbb{R}^{|\mathcal{V}|}$.
We use the mixed-logit loss
\begin{equation}
    \mathcal{L}_{\mathrm{mix}} = -\log P_{\mathrm{mix}}(y),
    \label{eq:appendix_ce}
\end{equation}
and the residual
\begin{equation}
    g = \nabla_{z_{\text{mix}}}\mathcal{L}_{\mathrm{mix}} = P_{\mathrm{mix}}(\cdot) - e_y.
    \label{eq:appendix_residual}
\end{equation}
For any negative token $k \neq y$, the margin is $m_k(z) = z[y] - z[k]$.
We use $\eta$ to denote the learning rate.
For each model $i \in \{\mathrm{weak}, \mathrm{strong}\}$, let $J_i = \partial z_i / \partial \theta_i$ and define the Jacobian Gram matrix
\begin{equation}
    K_i = J_i J_i^\top.
    \label{eq:appendix_gram}
\end{equation}

\begin{definition}[Centered logits and centered norm]
Let $\bar{z} = \frac{1}{|\mathcal{V}|}\mathbf{1}^\top z$ be the logit mean. The centered logits and centered norm are
\begin{equation}
    \tilde{z} = z - \bar{z}\mathbf{1}, \quad \|\tilde{z}\|_2 = \sqrt{\sum_{k=1}^{|\mathcal{V}|} (z[k] - \bar{z})^2}.
    \label{eq:appendix_centered_norm}
\end{equation}
Since $\|\tilde{z}\|_2 = \sqrt{|\mathcal{V}|} \cdot \text{Std}(z)$, the centered norm is a shift-invariant measure of sharpness.
\end{definition}

\subsection{First-Order Gradients and Centered Linearized Dynamics}
By the chain rule,
\begin{align}
    \nabla_{z_{\mathrm{weak}}} \mathcal{L}_{\mathrm{mix}} &= (1-\lambda) g, \label{eq:appendix_grad_z1} \\
    \nabla_{z_{\mathrm{strong}}} \mathcal{L}_{\mathrm{mix}} &= \lambda g. \label{eq:appendix_grad_z2}
\end{align}
Thus the mixed objective exposes branch-level logit derivatives in the same direction but with different magnitudes. The following branch-sensitivity calculation is a local diagnostic of how the jointly optimized weak--strong pair responds under the mixed objective. Under a first-order (local) linearization, we apply a Taylor expansion of logits around $\theta_i$:
\begin{equation}
    z_i(\theta_i + \Delta \theta_i) \approx z_i(\theta_i) + J_i \Delta \theta_i, \quad i \in \{\mathrm{weak}, \mathrm{strong}\}.
    \label{eq:appendix_taylor_logits}
\end{equation}
With $\Delta \theta_i = -\eta \nabla_{\theta_i}\mathcal{L}_{\mathrm{mix}}$, this yields
\begin{equation}
    \Delta z_i \approx -\eta J_i J_i^\top g = -\eta K_i g,
    \label{eq:appendix_update_z_i}
\end{equation}
and incorporating the mixing weights in \eqref{eq:appendix_grad_z1}--\eqref{eq:appendix_grad_z2} yields
\begin{equation}
    \Delta z_i \approx -\eta s_i K_i g, \quad s_{\mathrm{weak}} = 1-\lambda,\; s_{\mathrm{strong}} = \lambda.
    \label{eq:appendix_update_weighted}
\end{equation}

To remove mean-shift effects, let $\Pi = I - \frac{1}{|\mathcal{V}|}\mathbf{1}\mathbf{1}^\top$ be the centering projector. Since $\mathbf{1}^\top g = 0$, we have $\Pi g = g$ and define the centered kernel
\begin{equation}
    \tilde{K}_i = \Pi K_i \Pi.
    \label{eq:appendix_centered_kernel}
\end{equation}
Then the centered logit dynamics are
\begin{equation}
    \Delta \tilde{z}_i \approx -\eta s_i \tilde{K}_i g.
    \label{eq:appendix_update_centered}
\end{equation}

\begin{lemma}[PSD of centered kernel]
For each model $i \in \{\mathrm{weak}, \mathrm{strong}\}$, $\tilde{K}_i$ is positive semidefinite.
\end{lemma}
\begin{proof}
For any $x \in \mathbb{R}^{|\mathcal{V}|}$, $x^\top \tilde{K}_i x = x^\top \Pi J_i J_i^\top \Pi x = \|J_i^\top \Pi x\|_2^2 \ge 0$.
\end{proof}

For the jointly optimized training pair, the first-order fused-logit change is
\begin{align}
    \Delta \tilde{z}_{\text{mix}}
    &= (1-\lambda)\Delta \tilde{z}_{\mathrm{weak}} + \lambda \Delta \tilde{z}_{\mathrm{strong}} \nonumber\\
    &= -\eta \left[(1-\lambda)^2 \tilde{K}_{\mathrm{weak}} + \lambda^2 \tilde{K}_{\mathrm{strong}}\right] g.
    \label{eq:appendix_delta_mix}
\end{align}
The corresponding diagnostic first-order loss decrease follows from a Taylor expansion of $\mathcal{L}_{\mathrm{mix}}$ at $z_{\text{mix}}$:
\begin{equation}
    \Delta \mathcal{L}_{\mathrm{mix}} \approx \nabla_{z_{\text{mix}}}\mathcal{L}_{\mathrm{mix}}^\top \Delta \tilde{z}_{\text{mix}}
    = g^\top \Delta \tilde{z}_{\text{mix}}.
    \label{eq:appendix_taylor_loss}
\end{equation}
Substituting \eqref{eq:appendix_delta_mix} gives
\begin{equation}
    \Delta \mathcal{L}_{\mathrm{mix}} \approx g^\top \Delta \tilde{z}_{\text{mix}}
    = -\eta \left[(1-\lambda)^2 g^\top \tilde{K}_{\mathrm{weak}} g + \lambda^2 g^\top \tilde{K}_{\mathrm{strong}} g\right].
    \label{eq:appendix_delta_loss}
\end{equation}

\subsection{Stage I: Hard-Negative Amplification and Strong-Model Dominance}
Early in mixed-logit training, the weak model is more confused, so for many hard negatives $k$ we have $m_k(z_{\mathrm{weak}}) < m_k(z_{\mathrm{strong}})$. By the logit-mixing analysis in Section~\ref{sec:theoretical_analysis} of the main paper (Theorem~1), mixing shrinks these margins and increases total negative probability mass. Consequently, the residual $g$ grows in magnitude and is biased toward hard negatives: the weak model acts as a \emph{gradient amplifier}, not a competitor in final accuracy.

In this diagnostic decomposition, the effective per-step loss decrease attributable to model $i \in \{\mathrm{weak}, \mathrm{strong}\}$ is
\begin{equation}
    E_i \triangleq \eta s_i^2 g^\top \tilde{K}_i g, \quad s_{\mathrm{weak}} = 1-\lambda,\; s_{\mathrm{strong}} = \lambda.
\end{equation}
Dominance corresponds to $E_{\mathrm{strong}} > E_{\mathrm{weak}}$. Define the directional sensitivity $\kappa_i = g^\top \tilde{K}_i g / \|g\|_2^2$.

\begin{assumption}[Sensitivity advantage of the strong model]
Because the strong model is sharper (larger centered norm / lower entropy), its centered kernel responds more strongly along the residual direction:
\begin{equation}
    g^\top \tilde{K}_{\mathrm{strong}} g \approx \alpha \, g^\top \tilde{K}_{\mathrm{weak}} g, \quad \alpha > 1.
    \label{eq:appendix_alpha}
\end{equation}
\end{assumption}

Substituting \eqref{eq:appendix_alpha} into \eqref{eq:appendix_delta_loss} yields the total effective rate
\begin{equation}
    S(\lambda) = (1-\lambda)^2 + \alpha \lambda^2,
    \label{eq:appendix_total_rate}
\end{equation}
and the strong model dominates the effective mixed-logit update when
\begin{equation}
    \lambda^2 \alpha > (1-\lambda)^2.
    \label{eq:appendix_dominance}
\end{equation}
Solving \eqref{eq:appendix_dominance} gives the gradient-share crossover
\begin{equation}
    \lambda_{\text{cross}} = \frac{1}{1+\sqrt{\alpha}}.
    \label{eq:appendix_lambda_cross}
\end{equation}

\begin{remark}[Crossover vs.\ accuracy]
Equation~\eqref{eq:appendix_lambda_cross} characterizes a \emph{local} crossover of branch sensitivity under a one-step linearization. It does \emph{not} predict an accuracy inversion between the two agents. In practice, $g$, $\tilde{K}_{\mathrm{weak}}$, and $\tilde{K}_{\mathrm{strong}}$ evolve with $\lambda$ and training time, while optimizer dynamics smooth the trajectory, yielding a broad optimum region rather than a sharp transition.
\end{remark}

\begin{remark}[Softmax amplification]
Because $P_{\mathrm{mix}}(k) \propto \exp(z_{\text{mix}}[k])$, modest logit differences can disproportionately tilt $P_{\mathrm{mix}}$ and $g$. This amplifies hard-negative gradients and reinforces the early dominance in \eqref{eq:appendix_dominance}.
\end{remark}

\subsection{Stage II: Gradient Shielding via Hessian Contraction}
We next analyze why the weak-induced interaction becomes increasingly shielded once the strong branch becomes confident.

\subsubsection{Softmax Jacobian and Loss Hessian}
Recall $P_{\mathrm{mix}} = \text{Softmax}(z_{\text{mix}})$. The Jacobian of Softmax is
\begin{equation}
    \frac{\partial P_{\mathrm{mix}}}{\partial z_{\text{mix}}} = \text{diag}(P_{\mathrm{mix}}) - P_{\mathrm{mix}} P_{\mathrm{mix}}^\top.
    \label{eq:appendix_softmax_jacobian}
\end{equation}
Since $e_y$ is constant, the Hessian of the cross-entropy loss with respect to $z_{\text{mix}}$ is
\begin{equation}
    \mathbf{H}_{\mathcal{L}} = \nabla^2_{z_{\text{mix}}}\mathcal{L}_{\mathrm{mix}}
    = \frac{\partial (P_{\mathrm{mix}}-e_y)}{\partial z_{\text{mix}}}
    = \text{diag}(P_{\mathrm{mix}}) - P_{\mathrm{mix}} P_{\mathrm{mix}}^\top.
    \label{eq:appendix_loss_hessian}
\end{equation}

\begin{lemma}[PSD of the loss Hessian]
$\mathbf{H}_{\mathcal{L}}$ is positive semidefinite.
\end{lemma}
\begin{proof}
For any $v \in \mathbb{R}^{|\mathcal{V}|}$,
\begin{equation}
    v^\top \mathbf{H}_{\mathcal{L}} v
    = \sum_{j} P_{\mathrm{mix}}[j]\, v_j^2 - \left(\sum_j P_{\mathrm{mix}}[j]\, v_j\right)^2
    = \mathrm{Var}_{P_{\mathrm{mix}}}(v) \ge 0.
    \label{eq:appendix_hessian_psd}
\end{equation}
\end{proof}

\subsubsection{Interaction Hessian Between Models}
The mixed logits depend on both branches, so the cross-Hessian is useful as a sensitivity diagnostic: it captures how strong-logit changes contract the weak-side derivative, although the weak agent itself is not updated:
\begin{align}
    \mathbf{H}_{\mathrm{ws}}
    &= \frac{\partial}{\partial z_{\mathrm{strong}}}\left[\nabla_{z_{\mathrm{weak}}}\mathcal{L}_{\mathrm{mix}}\right] \nonumber\\
    &= (1-\lambda)\frac{\partial (P_{\mathrm{mix}}-e_y)}{\partial z_{\text{mix}}}\frac{\partial z_{\text{mix}}}{\partial z_{\mathrm{strong}}} \nonumber\\
    &= \lambda(1-\lambda)\mathbf{H}_{\mathcal{L}}.
    \label{eq:appendix_cross_hessian}
\end{align}

\subsubsection{Shielding in the Confident Regime}
Assume the target index is $y$ and the strong model drives the mixed prediction to $P_{\mathrm{mix}}(y) \to 1$. Then $P_{\mathrm{mix}}(k) \to 0$ for all $k \neq y$, which implies
\begin{equation}
    \lim_{P_{\mathrm{mix}} \to e_y}\mathbf{H}_{\mathcal{L}} = \mathbf{0}.
    \label{eq:appendix_hessian_collapse}
\end{equation}
Consequently,
\begin{equation}
    \lim_{P_{\mathrm{mix}} \to e_y} \mathbf{H}_{\mathrm{ws}} = \mathbf{0},
    \label{eq:appendix_cross_hessian_collapse}
\end{equation}
and the cross-branch curvature vanishes. This is the mathematical form of \emph{gradient shielding}: once the strong branch dominates and the fused prediction becomes highly concentrated, both the first-order residual $g$ and its local sensitivity collapse, making weak-induced corrections fade in the mixed objective.

\subsection{Stage III: Null-Space Drift and Mean Shift}
We now explain why global logit shifts can appear in post-WD-JT diagnostics without necessarily changing centered sharpness.

\subsubsection{Shift Invariance of the Mixed Softmax}
Softmax is invariant to global shifts. For any scalar $c$,
\begin{equation}
    \text{Softmax}(z + c\mathbf{1}) = \text{Softmax}(z).
    \label{eq:appendix_softmax_shift}
\end{equation}
Applying this to the fused logits, let $z_{\mathrm{weak}}' = z_{\mathrm{weak}} + c\mathbf{1}$. Then
\begin{align}
    z_{\text{mix}}'
    &= (1-\lambda) z_{\mathrm{weak}}' + \lambda z_{\mathrm{strong}} \nonumber\\
    &= z_{\text{mix}} + (1-\lambda) c \mathbf{1},
    \label{eq:appendix_shift_mix}
\end{align}
and by \eqref{eq:appendix_softmax_shift} the predictive distribution is unchanged. Hence, the loss is flat along the mean-shift direction of each model.

\subsubsection{Zero-Eigenvalue Direction of the Hessian}
Let $\mathbf{1}$ denote the all-ones vector. From \eqref{eq:appendix_loss_hessian},
\begin{equation}
    \mathbf{H}_{\mathcal{L}}\mathbf{1}
    = \text{diag}(P_{\mathrm{mix}})\mathbf{1} - P_{\mathrm{mix}}(P_{\mathrm{mix}}^\top \mathbf{1})
    = P_{\mathrm{mix}} - P_{\mathrm{mix}} = \mathbf{0}.
    \label{eq:appendix_null_direction}
\end{equation}
Thus $\mathbf{1}$ is a zero-eigenvalue direction of the Hessian, confirming that the loss has no curvature along global shifts.

\subsubsection{Why Drift Accumulates in the Null Space}
For cross-entropy, the gradient along a global-shift direction is exactly zero because $\mathbf{1}^{\top}g=0$. In any trainable branch, stochastic optimizer noise and optimizer regularization can still move parameters in directions that change raw logit means; a simplified stochastic component can be written as
\begin{equation}
    \theta^{+} = \theta - \eta \epsilon, \quad \epsilon \sim \mathcal{N}(0,\Sigma),
    \label{eq:appendix_sgd_stochasticity}
\end{equation}
which is a random walk. The logit space decomposes as
\begin{equation}
    z = \bar{z} \mathbf{1} + \tilde{z}, \quad
    \tilde{z}^\top \mathbf{1} = 0,
    \label{eq:appendix_decompose}
\end{equation}
where the null space is $\mathrm{span}\{\mathbf{1}\}$ and the active space is its orthogonal complement. In the active space, even small gradients can weakly restore the distribution shape. In the null space, there is no CE restoring force; stochasticity can increase the variance of $\bar{z}$, while optimizer details such as weight decay can determine the observed direction of the mean shift. This explains why raw logit means can drift without a corresponding increase in centered norm.

\subsection{Discussion: regime of validity, multi-step behaviour, and failure modes}
\label{app:theory_discussion}

\paragraph{Regime of validity.} The uniform margin shrinkage condition $m_k(z_{\mathrm{weak}}) \le m_k(z_{\mathrm{strong}})$ for \emph{all} $k \neq y$ (Theorem~1 of the main paper) is sufficient but not necessary. In practice, on a small fraction of tokens the strong agent may be \emph{less} confident than the weak agent (e.g., a token where the strong model has not yet committed). For such tokens, mixing shrinks negative mass instead of expanding it, and the global benefit of \modelname{} relies on the much larger pool of tokens where the assumption holds. Since $z_{\mathrm{mix}}$ is a convex combination, the effect on a violating token is bounded by the gap between $P_{\mathrm{strong}}$ and $P_{\mathrm{weak}}$ on that token, but the intended amplification can be reduced or reversed locally.

\paragraph{From single-step to multi-step training.}
The analysis in Section~\ref{sec:theoretical_analysis} of the main paper is a single-step linearisation: it characterises the immediate effect of one gradient update under fused logits. Across many steps, the cumulative behaviour is consistent with what we observe empirically. Stages~I and~II describe the typical regime in early to mid training, where amplification dominates and gradients on hard negatives keep the strong agent moving. Stage~III corresponds to late-training over-optimisation: as $P_{\mathrm{mix}}$ becomes nearly one-hot, gradients shrink on every token and optimizer dynamics can accumulate in the shift-invariant null space, producing raw mean-logit drift without comparable centered-norm growth. The convergence trajectories in Figure~\ref{fig:convergence} are consistent with this picture, but are used as diagnostics rather than as benchmark-specific early-stopping evidence.

\paragraph{When does suppression hurt?}
Theorem~1 and Stage~I together predict stronger suppression of non-target probability mass, consistent with the empirical $z_{\mathrm{bg}}\!:\, 2.09 \to 0.90$ in Section~\ref{sec:logit_statistics}. On tasks where the correct answer is \emph{single-mode}---competition math, code, structured logic---this suppression is desirable. On tasks where multiple valid completions coexist (open-ended generation, dialogue with multiple acceptable continuations, tasks requiring calibrated hedging), aggressive suppression of hard negatives may be over-confident. Our experiments cover only single-mode benchmarks; we leave hedging-sensitive evaluation to future work and flag this as a limitation in Appendix~\ref{app:limitations}.




\section{Data Construction and Training Details}
\label{app:dataconstruction}

\subsection{Evaluation Benchmarks and Baselines}
\label{app:eval_baselines}

\paragraph{Tasks and evaluation.} For mathematical reasoning, we report performance on AIME2025, MATH500, AMC23, AQuA~\cite{ling2017program}, GSM8K~\cite{cobbe2021training}, MAWPS~\cite{koncel2016mawps}, and SVAMP~\cite{patel2021nlp}. For code generation, we evaluate on HumanEval+ (HE+) and MBPP+, the augmented test-case versions of HumanEval~\cite{chen2021evaluatinglargelanguagemodels} and MBPP~\cite{austin2021program} introduced by EvalPlus~\cite{liu2023evalplus}, together with BigCodeBench (BCB)~\cite{zhuo2024bigcodebench} and LiveCodeBench (LCB)~\cite{jain2024livecodebench} for broader, contamination-resistant coverage. For logical reasoning, we evaluate on LogiQA~2.0~\cite{liu2020logiqa} and ReClor~\cite{yu2020reclor}.

\paragraph{Baselines.} \textbf{Standard SFT} optimizes the conventional next-token prediction objective without discrepancy-based sample selection or logit-space perturbation. \textbf{UNDIAL}~\citep{dong2024undialselfdistillationadjustedlogits} suppresses target-token logits using stochastic penalties to mitigate over-confident predictions. \textbf{NEFTune}~\citep{jain2023neftunenoisyembeddingsimprove} injects random noise into embeddings during training to improve generalization. \textbf{SPIN}~\citep{chen2024self} employs a self-play mechanism where the model iteratively distinguishes its own generated responses from human-annotated data, progressively converting a weak policy into a stronger one without external reward models. \textbf{SSB}~\citep{mitra2025semanticsoft} (Semantic Soft Bootstrapping) is an RL-free self-distillation method that uses the same base model as both teacher and student under different contexts: the teacher, conditioned on a hint trace, generates soft logit targets over the answer, which the student (without hints) learns to match via KL divergence, eliminating the need for preference data or RL training.

\subsection{Training Data Sources}
\label{app:data_sources}

Our training mixture covers three domains: mathematical reasoning, code generation, and logical reasoning. All samples use a chat message format (user/assistant turns), with assistant completions retaining the full \texttt{<think>}-wrapped chain of thought and a structured final answer. We use the full set of verified samples in each domain (Table~\ref{tab:train_data}); the math pool is by far the largest ($\sim$5.7$\times$ code, $\sim$7.5$\times$ logic).

\begin{table}[h]
\centering
\small
\caption{Training data across the three domains used in our experiments. Each domain uses the full set of verified samples from the corresponding source.}
\label{tab:train_data}
\begin{tabular}{l r p{8.0cm}}
\toprule
Domain & \#Samples & Source \\
\midrule
Math  & 111{,}657 & Full AM-DeepSeek-R1-1.4M~\cite{zhao202514millionopensourcedistilled} math subset, verified by \texttt{math\_verify} \\
Code  & 19{,}457  & AM-DeepSeek-R1-1.4M code subset, execution-verified against the provided unit tests \\
Logic & 14{,}882  & LogiQA~2.0 and ReClor distilled CoT, retained after answer-match verification (Section~\ref{app:logic_synth}) \\
\bottomrule
\end{tabular}
\end{table}

\paragraph{Mathematical reasoning.} Sourced from the AM-DeepSeek-R1-1.4M math subset ($\sim$111\,K verified examples spanning algebra, geometry, number theory, combinatorics, and competition problems). We retain DeepSeek-R1's full \texttt{<think>} CoT together with the final \texttt{\textbackslash boxed\{X\}} answer and discard samples that fail \texttt{math\_verify} parsing, leaving $111{,}657$ training examples. SSB and SPIN are exceptions and run on a $20\,$K subsample of this pool due to the high cost of their rollout phases.

\paragraph{Code generation.} Sourced from the same corpus, filtered by executing the produced code against the provided unit tests. We retain the \texttt{<think>} reasoning block; this yields $19{,}457$ training examples after execution verification.

\subsection{Logic Data Synthesis}
\label{app:logic_synth}

The Logic training set is synthesized in two steps from public benchmarks.

\paragraph{Step 1: merging raw questions.} We combine the raw splits of LogiQA~2.0~\cite{liu2020logiqa} (MRC version, \texttt{train}/\texttt{dev}/\texttt{test}) and ReClor~\cite{yu2020reclor} (\texttt{train}/\texttt{val}) into a unified JSON format ($\sim$17{,}205 items). Before any CoT distillation, the LogiQA~2.0 \texttt{test} split (1{,}572) and the ReClor \texttt{val} split (500) are removed from the training-candidate pool and held out for evaluation; we use ReClor \texttt{val} because the public ReClor \texttt{test} set has no released labels. The remaining training-candidate pool contains 15{,}133 items.

\paragraph{Step 2: CoT distillation with answer verification.} We prompt Qwen3-32B in its \texttt{thinking} mode on each training-candidate question to obtain a long-form chain-of-thought answer. A sample is retained only if the extracted prediction matches the ground-truth option. This yields 14{,}882 verified Logic training samples, corresponding to a $98.3\%$ retention rate after the held-out evaluation splits have already been removed.

\subsection{Gemma-3-4B-PT: Tag-Free Preprocessing}
\label{app:gemma_prep}

In preliminary experiments, Gemma-3-4B-PT did not reliably follow the \texttt{<think>}/\texttt{<answer>} tag convention during training. We therefore adopt a tag-free variant \emph{only for Gemma}.

\paragraph{Tag stripping.} Starting from the original CoT outputs, we remove the literal \texttt{<think>}, \texttt{</think>}, \texttt{<answer>}, and \texttt{</answer>} markers while retaining the reasoning text and the inner answer payload. Thus Gemma still trains on tag-free step-by-step solutions ending with \texttt{\textbackslash boxed\{X\}} for math, \texttt{The answer is (X)} for logic, or Python code blocks for code.

\paragraph{Fallback chat template.} Gemma-3-4B-PT ships without a chat template (only the IT variant has one). We inject a Gemma-3-IT-style fallback in which the \texttt{assistant} role is mapped to \texttt{model}, a leading system message is prepended to the first user turn (rather than forming its own turn), and each turn is wrapped with \texttt{<start\_of\_turn>}/\texttt{<end\_of\_turn>}. These special tokens are already present in the PT vocabulary.

\paragraph{Tagged vs.\ tag-free system prompts.} Qwen-series bases use a tag-style system prompt that instructs the model to emit \texttt{<think>$\cdots$</think>} followed by \texttt{<answer>$\cdots$</answer>}. Gemma uses a tag-free prompt that ends with ``Solve the problem step by step. End your solution with the final answer enclosed in \texttt{\textbackslash boxed\{\}}'' (math), ``End your response with `The answer is (X).'\,'' (logic), or ``Provide clean, correct Python code'' (code). The SPIN and SSB rollout/teacher phases are aligned to the base-appropriate prompt via their \texttt{--system-prompt} / \texttt{--student-system-prompt} flags so that the self-generated trajectories match the same formatting that the student is trained to emit.

\paragraph{vLLM multimodal wrap for evaluation.} Training saves a \texttt{Gemma3ForCausalLM} (text-only) checkpoint, while vLLM~v0.11 loads Gemma-3 only as \texttt{Gemma3ForConditionalGeneration} (multimodal). To make the saved checkpoint loadable for evaluation, we (i) re-key every weight with a \texttt{language\_model.}\ prefix and (ii) inject the base model's \texttt{vision\_tower.*} and \texttt{multi\_modal\_projector.*} weights unchanged. The wrapped checkpoint is functionally identical for text-only decoding but conforms to vLLM's expected module layout.

\subsection{Training Configuration}
\label{app:training_config}

\paragraph{Implementation details.}
We use Qwen3-4B-Base, Qwen3-8B-Base~\cite{yang2025qwen3technicalreport}, Qwen2.5-3B, and Gemma-3-4B-PT~\citep{gemmateam2025gemma3technicalreport} as backbone models, and fine-tune them on the AM-1.4M dataset~\cite{zhao202514millionopensourcedistilled} for math and code, together with the Logic training set synthesised in Appendix~\ref{app:logic_synth}. All experiments are conducted on eight NVIDIA H800 80GB GPUs. Our implementation is based on the \texttt{TRL} library~\citep{vonwerra2022trl} with the Qwen3 / Gemma architectures from \texttt{transformers}~\citep{wolf2020transformers}. We use full-parameter fine-tuning for SFT and \modelname{} unless a baseline explicitly states LoRA, AdamW optimization, a global learning rate of $1 \times 10^{-5}$, maximum sequence length $4096$, global batch size $512$, gradient clipping at $1.0$, and a $10\%$ warm-up ratio. For \modelname{}, we set $\alpha = 0.1$, $\beta = 0.8$, $\gamma = 0.1$, and use $\lambda = 0.5$ for the main runs unless otherwise stated; Appendix~\ref{app:hparam_protocol} reports a finer sensitivity sweep around this region.

\paragraph{Per-domain training budget.}
All methods reported in Table~\ref{tab:main_results} use the same base models and per-domain training budget unless an exception is stated below for rollout-heavy baselines. The total training budget is fixed to \textbf{2 epochs} for the main SFT-style comparisons, including the SFT baseline and \modelname{}.

\paragraph{Training schedule and early stopping.}
The SFT baseline is trained for 2 epochs. \modelname{} uses the same 2-epoch total budget: 1 SFT warm-up epoch producing $\mathcal{M}_1$, followed by 1 weak-driven epoch using WD-DS and WD-JT. This fixed budget is used consistently for the reported main comparisons rather than selected separately per benchmark. The convergence trajectories in Figure~\ref{fig:convergence} are reported as extended diagnostics beyond the main comparison setting: pushing further can introduce over-optimisation on small benchmarks (e.g.,\ AMC23) without improving the math average, consistent with the Stage~III drift in our theoretical analysis.

Exceptions specific to individual baselines are noted in the corresponding paragraph below.

\subsection{Baseline Hyperparameters}
\label{app:baseline_hparams}

We report the key hyperparameters we used for each baseline; all runs share the same base models and training data as \modelname, and follow the per-domain epoch budget described in Appendix~\ref{app:training_config} unless stated otherwise.

\paragraph{UNDIAL.} We follow the original recipe and set the logit-smoothing scale $\sigma=1.0$. Following common practice but not stated in the original paper, we normalize the injected noise by $1/\sqrt{L}$ (where $L$ is the sequence length) to avoid amplification on long sequences, so that \texttt{noise} $= \sigma\cdot\varepsilon/\sqrt{L}$ with $\varepsilon\sim\mathcal{N}(0,I)$.

\paragraph{NEFTune.} We adopt the paper-recommended embedding-noise magnitude $\alpha=5$.

\paragraph{SPIN.} SPIN is DPO-based self-play. Per the 2-epoch allocation in Appendix~\ref{app:training_config}, we initialise from the SFT-epoch-1 checkpoint and then run 1~epoch of DPO on SPIN-generated preference pairs. Hyperparameters were tuned on a small grid and we ended up using: $\beta = 0.1$ (recommended by SPIN/TRL's DPO and empirically preferred for iterative self-play, as low $\beta$ allows the student to deviate further from the reference), learning rate $5\times10^{-7}$ (DPO is highly learning-rate sensitive---$1\times10^{-6}$ was unstable in our runs), sampling temperature $1.0$ (following SPIN~\S4 for rejected-sample diversity), gradient-norm clip $0.5$ (tighter than the TRL default of $1.0$, which we found noticeably improves DPO stability) and a $10\%$ warm-up ratio to prevent early learning-rate shocks to the reference model. Following the original SPIN protocol, each iteration trains for exactly 1~epoch before resampling.

\paragraph{SSB.} SSB is a self-bootstrapping distillation baseline. For every prompt we first draw $4$ rollouts from the current student $m_2$ (sampling temperature $0.8$, top-$k$ $50$) and retain one correct and one incorrect trajectory. A teacher (the same $m_2$ conditioned on a hint-augmented prompt, temperature $0.4$, \texttt{max\_model\_len}$=16\,384$) then produces a corrective CoT. During optimisation we store the top-$20$ logits of each teacher token, which covers $>$95\% of the softmax mass, and train the student with a LoRA adapter ($r=32$, $\alpha=32$) using an equal mixture of KD ($T=1.5$) and cross-entropy loss. Training uses $\mathrm{lr}=1\times10^{-4}$, 2~epochs, and a global batch size of $256$ ($1\times32\times8$). For math, SSB's rollout and teacher-generation phases are substantially more expensive than the other baselines, so we subsample to 20K examples while keeping the 2-epoch budget; this subsampling applies only to the SSB math setting.

\paragraph{SSB on Gemma-3-4B-PT.} Because the Gemma pretrained model alone cannot produce coherent rollouts for the self-play stage, we additionally initialise the SSB student from a 1-epoch SFT checkpoint and then run one further SSB epoch on top (total budget still $=$ 2 epochs). The Qwen-series bases (Qwen2.5-3B, Qwen3-4B-Base, Qwen3-8B-Base) do not require this warm-up and run SSB directly from the base model.

%% file: table/main_results_more_models.tex
\begin{table}[H]
    \centering
    \small
    \setlength{\tabcolsep}{4pt}
    \caption{\textbf{Results on Gemma-3-4B-PT.} We extend the main-paper benchmark (Table~\ref{tab:main_results}) to a different model family. All numbers are pass@1 accuracy (\%) under greedy decoding.}
    \label{tab:main_results_more_models}
    \resizebox{\textwidth}{!}{%
    \begin{tabular}{@{}l|c c c c c c c|c c c c c@{}}
        \toprule
        \multirow{2}{*}{\textbf{Methods}} & \multicolumn{7}{c|}{\textbf{Math}} &  \multicolumn{5}{c}{\textbf{Code}} \\
        \cmidrule(l){2-8}\cmidrule(l){9-13}
         & \textbf{MATH500} & \textbf{AMC23} & \textbf{AQUA} & \textbf{GSM8K} & \textbf{MAWPS} & \textbf{SVAMP}& Avg. & \textbf{HE+} & \textbf{MBPP+} & \textbf{BCB} & \textbf{LCB} & Avg. \\
        \midrule
         SFT & 28.0 & 17.5 & 46.5 & 54.0 & 83.6 & 68.3 & 49.7 & 36.6 & 43.1 & 18.5 & 7.5 & 26.4\\
         UNDIAL & 31.8 & \underline{20.0} & 43.7 & 58.8 & \underline{86.6} & \textbf{70.9} & 52.0 & 36.6 & \underline{47.1} & 19.2 & 7.8 & 27.7\\
         NEFTune & \underline{32.6} & 17.5 & 46.1 & \underline{59.2} & \underline{86.6} & \underline{70.8} & \underline{52.1} & 37.8 & 46.3 & \underline{20.1} & \underline{8.0} & \underline{28.1}\\
         SPIN & 30.6 & 7.5 & \underline{48.0} & 58.6 & 84.5 & 66.1 & 49.2 & 37.2 & 34.7 & 19.6 & 7.5 & 24.8\\
         SSB & 31.2 & 5.0 & 43.7 & 51.1 & 69.3 & 58.7 & 43.2 & \underline{39.0} & 39.7 & 19.0 & 7.3 & 26.3\\
         \textbf{\modelname} & \textbf{33.2} & \textbf{22.5} & \textbf{51.2} & \textbf{60.5} & \textbf{87.4} & 69.5 & \textbf{54.1} & \textbf{41.5} & \textbf{48.9} & \textbf{21.8} & \textbf{9.0} & \textbf{30.3}\\
        \bottomrule
    \end{tabular}%
    }
\end{table}

%% file: table/logic_results.tex
\begin{table*}[!htbp]
    \centering
    \small
    \setlength{\tabcolsep}{4pt}
    \caption{\textbf{Logic-reasoning results across base models.} We evaluate \modelname against SFT, UNDIAL, NEFTune, SPIN, and SSB on LogiQA~2.0 and ReClor. All numbers are pass@1 accuracy (\%) under greedy decoding. \textbf{Bold} marks the best value per column; \underline{underline} marks the second best.}
    \label{tab:logic_results}
    \resizebox{\textwidth}{!}{%
    \begin{tabular}{@{}l|c c c|c c c|c c c|c c c@{}}
        \toprule
        \multirow{2}{*}{\textbf{Methods}} & \multicolumn{3}{c|}{\textbf{Qwen3-4B-Base}} & \multicolumn{3}{c|}{\textbf{Qwen3-8B-Base}} & \multicolumn{3}{c|}{\textbf{Qwen2.5-3B}} & \multicolumn{3}{c}{\textbf{Gemma-3-4B-PT}} \\
        \cmidrule(lr){2-4} \cmidrule(lr){5-7} \cmidrule(lr){8-10} \cmidrule(lr){11-13}
         & \textbf{LogiQA2} & \textbf{ReClor} & Avg. & \textbf{LogiQA2} & \textbf{ReClor} & Avg. & \textbf{LogiQA2} & \textbf{ReClor} & Avg. & \textbf{LogiQA2} & \textbf{ReClor} & Avg. \\
        \midrule
        SFT      & 66.1 & 74.6 & 70.4 & 69.2 & 79.8 & 74.5 & 48.5 & 52.0 & 50.3 & 55.4 & 63.3 & 59.4 \\
        UNDIAL   & \underline{69.1} & 75.4 & 72.3 & 71.3 & 81.8 & 76.6 & \underline{54.8} & \underline{60.0} & \underline{57.4} & 55.7 & 67.4 & 61.6 \\
        NEFTune  & 68.2 & 73.0 & 70.6 & \underline{72.3} & \underline{83.8} & \underline{78.1} & 54.1 & 59.0 & 56.6 & 56.1 & 68.2 & 62.2 \\
        SPIN     & 68.3 & \underline{76.8} & \underline{72.6} & 71.3 & 80.6 & 76.0 & 51.5 & 58.8 & 55.2 & \underline{58.1} & \underline{71.0} & \underline{64.6} \\
        SSB      & 67.5 & 74.6 & 71.1 & 69.0 & 81.0 & 75.0 & 50.8 & 53.8 & 52.3 & 57.1 & 70.8 & 64.0 \\
        \textbf{\modelname} & \textbf{69.5} & \textbf{79.2} & \textbf{74.4} & \textbf{74.6} & \textbf{89.2} & \textbf{81.9} & \textbf{58.0} & \textbf{67.6} & \textbf{62.8} & \textbf{60.5} & \textbf{75.6} & \textbf{68.1} \\
        \bottomrule
    \end{tabular}%
    }
\end{table*}

%% file: table/logit_analysis.tex
\begin{table}[H]
    \centering
    \small 
    \setlength{\tabcolsep}{4pt} 
    \caption{\textbf{Logit Dynamics Analysis (Qwen3-4B-Base).} Comparison of logit statistics at Epoch 3. Values in parentheses indicate the relative change of \modelname vs. SFT.}
    \label{tab:logit_analysis}
    
    \renewcommand{\arraystretch}{1.25}
    
    \begin{tabular}{l c c}
        \toprule
        \textbf{Metric} & \textbf{SFT} & \textbf{\modelname} \\
        \midrule
        
        Target Strength ($z_{\text{target}}$) & $35.88$ & $36.10$ {\scriptsize(\textcolor{blue}{+0.6\%})} \\
        
        Distractors Mean ($z_{\text{bg}}$) & $2.09$ & $\mathbf{0.90}$ {\scriptsize(\textcolor{red}{\textbf{-56.9\%}})} \\
        
        \midrule
        
        Target-to-Background Gap ($\Delta_{\text{gap}}$) & $33.79$ & $\mathbf{35.20}$ {\scriptsize(\textcolor{teal}{+4.2\%})} \\
        
        Logit Std. ($\sigma$) & $2.93$ & $3.45$ {\scriptsize(\textcolor{teal}{+17.7\%})} \\
        
        \bottomrule
    \end{tabular}
\end{table}

%% file: table/logit_analysis_extended.tex
\begin{table}[t]
    \centering
    \caption{\textbf{Extended Logit Statistics Before vs.\ After WD-JT (Qwen3-4B-Base).} The weak column reports the weak-side initialization used to form mixed logits; pre/post statistics are reported for the retained strong branch. $\Delta$ denotes strong Post$-$Pre (percent change in parentheses). Centered norm values are taken directly from the logit analysis report.}
    \label{tab:logit_analysis_extended}
    \scriptsize
    \setlength{\tabcolsep}{4pt}
    \renewcommand{\arraystretch}{1.15}
    \begin{tabular}{l c c c c}
        \toprule
        Metric & Weak init & Strong Pre & Strong Post & $\Delta$ \\
        \midrule
        Mean logit $z_{\text{mean}}$ & 3.65 & 2.87 & 0.97 & -1.90 (-66.2\%) \\
        Std $\sigma$ & 3.06 & 3.18 & 3.16 & -0.02 (-0.6\%) \\
        Centered norm $\|\tilde{z}\|_2$ & 1191.33 & 1240.10 & 1229.79 & -10.31 (-0.8\%) \\
        Max logit $z_{\max}$ & 43.25 & 48.75 & 60.50 & 11.75 (+24.1\%) \\
        Min logit $z_{\min}$ & -24.00 & -26.63 & -35.75 & -9.12 (+34.2\%) \\
        L2 norm $\|z\|_2$ & 1978.92 & 1865.96 & 1688.59 & -177.37 (-9.5\%) \\
        Entropy $H$ & 0.52 & 0.44 & 0.48 & 0.04 (+9.1\%) \\
        Max prob $p_{\max}$ & 0.83 & 0.85 & 0.83 & -0.02 (-2.4\%) \\
        \bottomrule
    \end{tabular}
\end{table}

%% file: table/hyperparam_sensitivity.tex
\begin{table}[t]
    \centering
    \small
    \setlength{\tabcolsep}{4pt}
    \caption{\textbf{Post-hoc hyperparameter sensitivity analysis.} We evaluate the impact of the mixing coefficients $\alpha$ (consolidation), $\beta$ (base difficulty), and $\gamma$ (regression repair) on mathematical reasoning across two base models. These sweeps are diagnostic and are not used to select the headline models in Table~\ref{tab:main_results}. All numbers are pass@1 accuracy (\%) under greedy decoding. \textbf{Bold} marks the best value per metric within each base-model column group. ``-'' indicates that the base model produces no correct answer under any configuration.}
    \label{tab:hyperparam_sensitivity}
    \begin{tabular}{l c c c c c c c}
        \toprule
        \multirow{2}{*}{\textbf{Setup}} & \multicolumn{3}{c}{\textbf{Coefficients}} & \multicolumn{2}{c}{\textbf{Qwen3-4B-Base}} & \multicolumn{2}{c}{\textbf{Qwen2.5-3B}} \\
        \cmidrule(lr){2-4} \cmidrule(lr){5-6} \cmidrule(lr){7-8}
         & $\alpha$ & $\beta$ & $\gamma$ & \textbf{AIME 2025} & \textbf{MATH 500} & \textbf{AIME 2025} & \textbf{MATH 500} \\
        \midrule
        Config A & 0.2 & 0.7 & 0.1 & 10.0 & 68.2 & - & 44.6 \\
        \rowcolor{gray!15} \textbf{Default} & \textbf{0.1} & \textbf{0.8} & \textbf{0.1} & \textbf{16.7} & 68.2 & - & \textbf{48.2} \\
        Config B & 0.2 & 0.6 & 0.2 & 10.7 & 67.8 & - & 42.0 \\
        Config C & 0.1 & 0.9 & 0.0 & 10.3 & \textbf{70.2} & - & 47.6 \\
        Config D & 0.0 & 1.0 & 0.0 & 9.7 & 68.3 & - & 47.0 \\
        Config E & 0.05 & 0.85 & 0.1 & 14.0 & 69.5 & - & 47.2 \\
        \bottomrule
    \end{tabular}
\end{table}

%% file: table/all_lambda.tex
\begin{table*}[t]
    \centering
    \caption{\textbf{Post-hoc sensitivity analysis of mixing coefficient $\lambda$.} We report average greedy pass@1 accuracy (\%) over three independent training runs per $\lambda$ on Qwen3-4B-Base. These sweeps are diagnostic and are not used to select the headline models in Table~\ref{tab:main_results}, which use the fixed default $\lambda=0.5$. The curve exhibits a broad inverted U-shape, with the strongest performance typically in the $\lambda \in [0.42, 0.50]$ range (peak at $\lambda=0.42$ in Avg.), reflecting a balanced trade-off between learning new features and preserving historical constraints.}
    \label{tab:sensitive_of_lambda}
    
    \setlength{\tabcolsep}{4pt}
    
    \begin{small}
    \resizebox{\textwidth}{!}{
        \begin{tabular}{l ccccccccccccccccc}
            \toprule
            \multirow{2}{*}{\textbf{Dataset}} & \multicolumn{17}{c}{\textbf{Mixing Coefficient} $\lambda$} \\
            \cmidrule(l){2-18}
            & 0.1 & 0.2 & 0.3 & 0.4 &\textbf{0.42} & 0.44 & 0.46 & 0.48 & 0.5 & 0.52 & 0.54 & 0.56 & 0.58 & 0.6 & 0.7 & 0.8 & 0.9 \\
            \midrule
            AIME25   &  7.8 & 10.0 & 12.2 & 16.7 & \textbf{20.0} & 20.0 & 16.7 & 20.0 & 20.0 & 21.1 &  16.7 & 16.7 & 16.7 & 10.0 & 20.0 & 20.0 & 12.2\\
            MATH500 & 71.4 & 67.1 & 70.8 & 73.1 & \textbf{74.9} & 74.5 & 73.2 & 73.3 & 73.3 & 70.7 &  74.0 & 73.0 & 73.3 & 68.9 & 71.7 & 70.9 & 66.1\\
            AQUA  &  77.2 & 72.8 & 74.7 & 75.2 & \textbf{73.9} & 73.1 & 72.4 & 71.8 & 73.9 & 71.0 &  71.3 & 71.9 & 70.2 & 64.8 & 66.9 & 61.4 & 60.8 \\
            GSM8K  & 88.2 & 89.1 & 89.4 & 91.9 & \textbf{91.8} & 91.7 & 91.0 & 91.4 & 91.0 & 91.2 &  91.4 & 90.6 & 90.9 & 89.2 & 89.3 & 88.1 & 85.7 \\
            MAWPS  & 94.9 & 96.6 & 95.5 & 97.8 & \textbf{98.2} & 97.8 & 97.8 & 98.0 & 95.7 & 96.5 &  97.2 & 96.8 & 97.5 & 94.7 & 96.2 & 95.5 & 93.6\\
            SVAMP  & 90.4 & 91.8 & 93.3 & 94.9 & \textbf{94.1} & 93.6 & 93.5 & 92.9 & 92.6 & 92.4 & 93.5 & 92.8 & 92.9 & 90.8 & 91.9 & 89.6 & 87.4\\
            \midrule
            \textit{Avg.} & 71.7 & 71.2 & 72.7 & 74.9 & \textbf{75.5} & 75.1 & 74.1 & 74.6 & 74.4 & 73.8 & 74.0 & 73.6 & 73.6 & 69.7 & 72.7 & 70.9 & 67.6\\
            \bottomrule
        \end{tabular}
    }
    \end{small}
\end{table*}